\documentclass[sigconf]{acmart}

\usepackage{caption, subcaption, xcolor, amsmath}

\newcommand{\retime}{{{\textsc{ReTime}}}}

\newtheorem{mydef}{Definition}[section]

\AtBeginDocument{%
  \providecommand\BibTeX{{%
    \normalfont B\kern-0.5em{\scshape i\kern-0.25em b}\kern-0.8em\TeX}}}

\setcopyright{acmcopyright}
\copyrightyear{2022}
\acmYear{2022}
\acmDOI{XXXXXXX.XXXXXXX}

\acmConference[CIKM'22 AMLTS]{CIKM workshop on Applied Machine Learning Methods for Time Series Forecasting}{October 17--21, 2022}{Atlanta, GA, USA.}
\acmBooktitle{CIKM'22 workshop on Applied Machine Learning Methods for Time Series Forecasting (AMLTS), October 17--21, 2022, Atlanta, GA, USA}

\acmPrice{15.00}
\acmISBN{978-1-4503-XXXX-X/18/06}

\begin{document}

\title{Retrieval Based Time Series Forecasting}


\author{Baoyu Jing}
\email{baoyuj2@illinois.edu}
\affiliation{%
  \institution{University of Illinois at Urbana-Champaign}
    \country{}
}

\author{Si Zhang}
\email{sizhang@fb.com}
\affiliation{%
  \institution{Meta}
    \country{}
}
\author{Yada Zhu}
\email{yzhu@us.ibm.com}
\affiliation{%
  \institution{IBM Research}
    \country{}
}
\author{Bin Peng}
\email{binpeng@illinois.edu}
\affiliation{%
  \institution{University of Illinois at Urbana-Champaign}
    \country{}
}
\author{Kaiyu Guan}
\email{kaiyug@illinois.edu}
\affiliation{%
  \institution{University of Illinois at Urbana-Champaign}
    \country{}
}
\author{Andrew Margenot}
\email{margenot@illinois.edu}
\affiliation{%
  \institution{University of Illinois at Urbana-Champaign}
    \country{}
}
\author{Hanghang Tong}
\email{htong@illinois.edu}
\affiliation{%
  \institution{University of Illinois at Urbana-Champaign}
    \country{}
}

\renewcommand{\shortauthors}{Baoyu Jing et al.}


\begin{abstract}

Time series data appears in a variety of applications such as smart transportation and environmental monitoring.
One of the fundamental problems for time series analysis is time series forecasting. 
Despite the success of recent deep time series forecasting methods, they require sufficient observation of historical values to make accurate forecasting. 
In other words, the ratio of the output length (or forecasting horizon) to the sum of the input and output lengths should be low enough (e.g., 0.3).
As the ratio increases (e.g., to 0.8), the uncertainty for the forecasting accuracy increases significantly.
In this paper, we show both theoretically and empirically that the uncertainty could be effectively reduced by retrieving relevant time series as references. 
In the theoretical analysis, we first quantify the uncertainty and show its connections to the Mean Squared Error (MSE). 
Then we prove that models with references are easier to learn than models without references since the retrieved references could reduce the uncertainty.
To empirically demonstrate the effectiveness of the retrieval based time series forecasting models, we introduce a simple yet effective two-stage method, called \retime\, consisting of a relational retrieval and a content synthesis. 
We also show that \retime\ can be easily adapted to the spatial-temporal time series and time series imputation settings.
Finally, we evaluate \retime\ on real-world datasets to demonstrate its effectiveness. 

\end{abstract}




\maketitle

\section{Introduction}
Time series analysis has received paramount interest in numerous real-world applications \cite{zhou2021informer,gamboa2017deep,li2019enhancing,wu2021autoformer,DBLP:conf/www/ZhouZZLH20, zhu2022network, yang2019deep}, such as smart transportation and environmental monitoring.
Accurately forecasting time series provides valuable insights for making transport policies \cite{li2018exploring} and climate change policies \cite{sarker2012exploring}.

One of the fundamental problems for time series analysis is time series forecasting.
Despite the success of recent deep learning methods for time series forecasting \cite{zhou2021informer,gamboa2017deep,li2019enhancing,wu2021autoformer,DBLP:conf/www/ZhouZZLH20},
a sufficient observation of the input historical time series is required.
Specifically, the \textit{ratio} of the output length (i.e., forecasting horizon) to the sum of the input and output length should be sufficiently low (e.g., 0.3).
In fact, this ratio is a special case of the \textit{missing rate} used in time series imputation, and thus for clarity and consistency with the literature, we formulate the task of time series forecasting from the perspective of time series imputation in this paper. 
Please refer to Section \ref{sec:preliminary} for details.
When the missing rate increases to a high level (e.g., 0.8), the uncertainty of the forecasting accuracy will increase significantly.
In real-world applications, it is common that one wants to forecast future values based on very limited observations. 
In smart transportation, government administrators might be interested in forecasting the traffic conditions of a road with broken sensors \cite{wang2009forecasting}. 
In environmental monitoring, geologists always have a desire of obtaining the temperature of a specific location, where sensors cannot be easily placed \cite{nalder1998spatial}.

To address this problem, we first formally quantify the uncertainty of the forecasting based on the entropy of the ground truth conditioned on the predicted values, and show its connections to the Mean Squared Error (MSE).
Then we theoretically prove that models with reference time series are easier to train than those without references, since the references could reduce uncertainties.
Motivated by the theoretical analysis, we introduce a simple yet effective two-stage time forecasting method called \retime.
Given a target time series, in the first relational retrieval stage, \retime\ retrieves the references from a database based on the relations among time series.
We use relational retrieval rather than content based retrieval since the input historical values of the target time series could be very unreliable when the input length is very short.
Thus, content-based methods could retrieve unreliable references.
In comparison, the relational information is usually reliable and easy to obtain in practice \cite{bonnet2001towards, pelkonen2015gorilla, rhea2017littletable, song2018deep}, such as whether two sensors are adjacent to each other in traffic/environmental monitoring.
In the second content synthesis stage, \retime\ synthesizes the future values based on the content of the target and the references.
Next, we show that the proposed \retime\ could be easily applied to the time series imputation task and the spatial-temporal time series setting.
Finally, we empirically evaluate \retime\ on two real-world datasets to demonstrate its effectiveness.

The main contributions of the paper are summarized as follows: 
\begin{itemize}
    \item \textbf{Theoretical Analysis.} We theoretically quantify the uncertainty of the predicted values based on conditional entropy and show its connections to the MSE. We also theoretically demonstrate that models with references are easier to train than those without references.
    \item \textbf{Algorithm.} We introduce a two-stage method \retime\ for time series forecasting, which is comprised of relational retrieval and content synthesis. \retime\ can also be easily applied to the spatial-temporal time series and time series imputation settings.
    \item \textbf{Empirical Evaluation.} We evaluate \retime\ on two real-world datasets and various settings (i.e., single/spatial temporal forecasting/imputation) to demonstrate its effectiveness. 
\end{itemize}

\section{Preliminary}\label{sec:preliminary}
The ratio ${L_{out}}/{(L_{in}+L_{out})}$ of the output length $L_{out}$ to the sum of the input length $L_{in}$ and output length $L_{out}$ is a special case of the missing rate used in time series imputation, and thus we formulate tasks of time series forecasting from the perspective of time series imputation.
We summarize the mathematical notations in Table \ref{tab:notations}.

\begin{mydef}[Single Time Series Forecasting]
Given an incomplete target time series $\mathbf{X}\in\mathbb{R}^{T\times v}$, where $T$ and $v$ are the numbers of time steps and variates, along with its indicator mask $\mathbf{M}\in\{0,1\}^{{T}\times v}$, which indicates the absence/presence of the data points, the task aims to generate a new target time series $\hat{\mathbf{X}}\in\mathbb{R}^{T\times v}$ to predict the missing values $(1-\mathbf{M})\odot\hat{\mathbf{X}}$.
\end{mydef}

\begin{mydef}[Retrieval Based Single Time Series Forecasting]
Given an incomplete target time series $\mathbf{X}\in\mathbb{R}^{T\times v}$, where $T$ and $v$ are the numbers of time steps and variates, along with its indicator mask $\mathbf{M}\in\{0,1\}^{{T}\times v}$, the task aims to generate a new target time series $\hat{\mathbf{X}}\in\mathbb{R}^{T\times v}$ to predict the missing values $(1-\mathbf{M})\odot\hat{\mathbf{X}}$ based on the input target time series $\mathbf{X}$ and the $K$ reference time series $\{\mathbf{Y}_k\}_{k=1}^K$ retrieved from a database $\{\mathbf{Y}_n'\in\mathbb{R}^{T'\times v}\}_{n=1}^N$, where $N\gg K$ is the number of time series and $T'\gg T$.
\end{mydef}

\begin{mydef}[Spatial-Temporal Time Series Forecasting]
Given an incomplete spatial-temporal time series $\mathbf{X}\in\mathbb{R}^{N\times T \times v}$, where $N$, $T$ and $v$ are the numbers of time series, time steps, and variates, along with its indicator mask $\mathbf{M}\in\{0,1\}^{N\times T \times v}$, where 0/1 indicates the absence/presence of the data points, and its adjacency matrix $\mathbf{A}\in\mathbb{R}^{N\times N}$,
the task is to generate a new time series $\hat{\mathbf{X}}\in\mathbb{R}^{N\times T \times v}$ to predict the missing values $(1-\mathbf{M})\odot\hat{\mathbf{X}}$.
\end{mydef}

Note that for forecasting, missing points and zeros are concentrated at the end of the target $\mathbf{X}$ and mask $\mathbf{M}$ after a certain separation time step $\tau$: $\mathbf{M}[\tau:T]=0$.

\begin{table}[]
    \centering
    \begin{tabular}{ll}
    \hline
    Notation &  Description\\
    \hline
    $\mathbf{X}$ & incomplete target time series\\
    $\Tilde{\mathbf{X}}$ & complete target time series\\
    $\hat{\mathbf{X}}$ & predicted target time series\\
    $\mathbf{Y}_k$ & $k$-th reference time series\\
    $\mathbf{Y}_n'$ & $n$-th time series in the database\\
    $\mathbf{M}$ & indicator mask of the target time series\\
    $\mathbf{A}'$ & adjacency matrix for the database $\{\mathbf{Y}_n'\}_{n=1}^N$\\
    $\mathbf{A}$ & adjacency matrix for $\{\mathbf{Y}_n'\}_{n=1}^N$ and $\mathbf{X}$\\
    $\mathcal{R}$ & relation between $\mathbf{X}$ and $\{\mathbf{Y}_n'\}_{n=1}^N$\\
    $T$ & length of time series\\
    $T'$ & length of time series in $\{\mathbf{Y}_n'\}_{n=1}^N$\\
    $N$ & number of time series in the database\\
    $K$ & number of reference time series\\
    $v$ & number of variates\\
    $d$ & size of hidden dimension\\
    $\tau$ & separation time separating history and future time\\
    \hline
    \end{tabular}
    \caption{Mathematical Notations}
    \label{tab:notations}
\end{table}

\section{Theoretical Motivation}\label{sec:theory}
\begin{figure}
    \centering
    \begin{subfigure}[b]{.2\textwidth}
        \centering
        \includegraphics[width=0.8\textwidth]{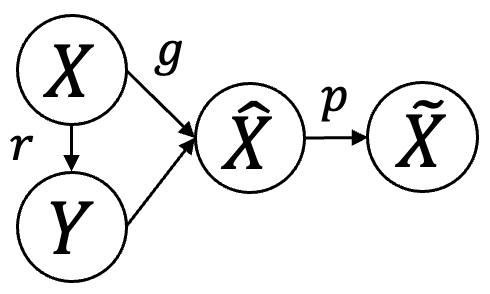}
        \caption{With references}\label{fig:ref}
    \end{subfigure}
    \quad
    \begin{subfigure}[b]{.2\textwidth}
    \centering
        \includegraphics[width=0.8\textwidth]{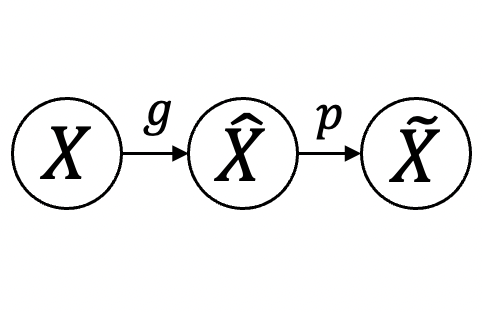}
        \caption{Without references}\label{fig:no_ref}
    \end{subfigure}
    \caption{Graphical models for methods w. or w/o references. $X$, $\hat{X}$, $\Tilde{X}$ and $Y$ are random variables for incomplete, generated, complete, and retrieved time series. $g$ and $r$ are the generation and retrieval models. $p$ is the relation between $\hat{X}$ and $\Tilde{X}$.}
    \label{fig:graphical_model}
\end{figure}

\begin{figure*}
    \centering
    \includegraphics[width=0.7\textwidth]{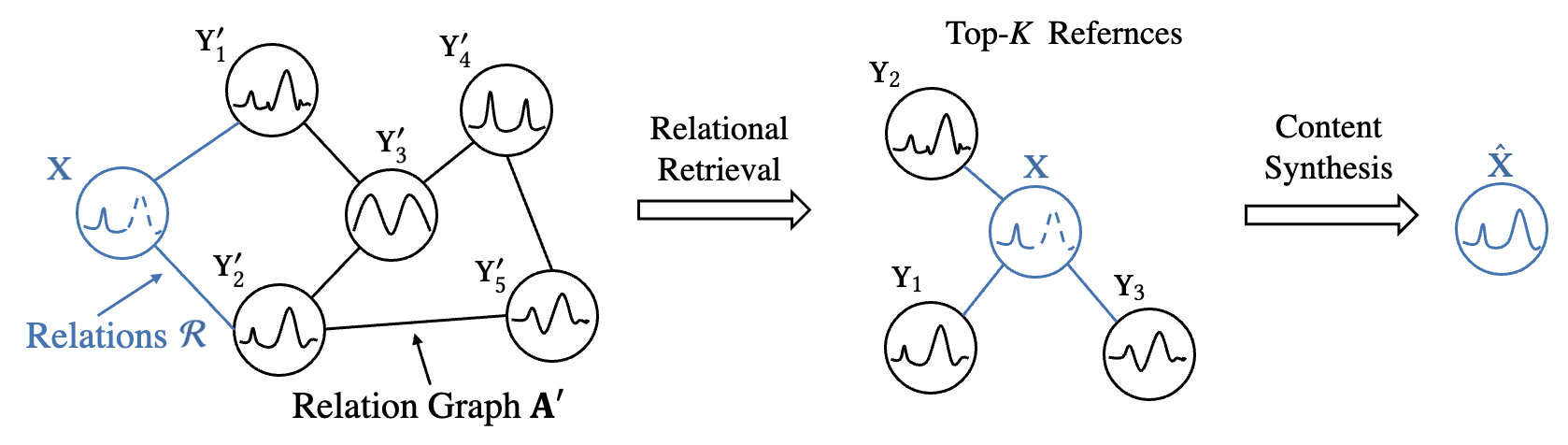}
    \caption{Overview of \retime. Given $\mathbf{X}$ and $\{\mathbf{Y}_{n}'\}_{n=1}^N$, \retime\  first retrieves the top $K$ references $\{\mathbf{Y}_k\}_{k=1}^K$ based on the relations $\mathbf{A}'$ and $\mathcal{R}$, and then combines the content of $\mathbf{X}$ and $\{\mathbf{Y}_{k}\}_{k=1}^K$ to generate $\hat{\mathbf{X}}$. Solid/dashed curves are observed/unobserved values.}
    \label{fig:overview}
\end{figure*}

An illustration of graphical models for methods with or without references is presented in Figure \ref{fig:graphical_model}, which shows the relations among random variables. 
$X$, $\Tilde{X}$, $\hat{X}$ and $Y$ denote the random variables for the incomplete target, complete target, generated target, and retrieved reference time series respectively.
$g$ and $r$ denote the generation and retrieval model respectively. 
$p$ is the relation between $\hat{X}$ and $\Tilde{X}$.
Following the common practice for linear regression \cite{murphy2012machine}, which assumes a Gaussian noise between $\hat{X}$ and $\Tilde{X}$, we define the relation $p$ as:
\begin{equation}\label{eq:relation}
    p(\Tilde{\mathbf{x}}|\hat{\mathbf{x}})=\mathcal{N}(\Tilde{\mathbf{x}}|\hat{\mathbf{x}}, \sigma^2\mathbf{I})
\end{equation}
where $\mathcal{N}$ denotes the Gaussian distribution, $\hat{\mathbf{x}}$ and $\Tilde{\mathbf{x}}\in\mathbb{R}^{v}$ denote the values of $\hat{\mathbf{X}}$ and $\Tilde{\mathbf{X}}\in\mathbb{R}^{(T-\tau)\times v}$ at a future time step $t\in[\tau, T]$,
$\sigma$ is the standard deviation, and 
$\mathbf{I}\in\mathbb{R}^{v\times v}$ is the identity matrix.
Note that $\mathbf{X}[:\tau]=\Tilde{\mathbf{X}}[:\tau]$ for historical steps $t\in[1, \tau)$, and thus we only consider the future time steps $t\in[\tau, T]$.

We first quantify the uncertainty $\Delta$ for the accuracy of $\hat{X}$ in Definition \ref{def:uncertainty} as the entropy of $\Tilde{X}$ conditioned on $\hat{X}$.
According to Equation \ref{eq:relation} and the definition of conditional entropy, we can calculate the uncertainty $\Delta$ (Lemma \ref{lemma:uncertainty_calculation}). 
As the only parameter in $\Delta$ is the standard deviation $\sigma$ in Equation \ref{eq:relation},  we can further prove that $\Delta$ is equivalent to the MSE between $\Tilde{X}$ and $\hat{X}$.

\begin{mydef}[Uncertainty of $\hat{X}$]\label{def:uncertainty}
\begin{equation}
    \Delta = H(\Tilde{X}|\hat{X})
\end{equation}
where $H$ denotes the entropy.
\end{mydef}

\begin{lemma}[Uncertainty Calculation]\label{lemma:uncertainty_calculation}
According to the definition of conditional entropy and Equation \eqref{eq:relation}, we have:
\begin{equation}
    \Delta = \frac{v}{2}(1+\log2\pi\sigma^2)
\end{equation}
\end{lemma}

\begin{lemma}[Equivalence between Uncertainty and MSE]\label{lemma:delta=mse}
The uncertainty is equivalent to MSE.
\begin{equation}
    \Delta \Leftrightarrow MSE
\end{equation}
\end{lemma}
\begin{proof}
The only parameter of $\Delta$ in Lemma \ref{lemma:uncertainty_calculation} is the standard deviation $\sigma$, which can be estimated by:
\begin{equation}
    \sigma = \sqrt{\frac{1}{Z}\sum_{z=1}^Z||\Tilde{\mathbf{x}}_z-\hat{\mathbf{x}}_z||^2}
\end{equation}
where $Z$ is the total number of data pairs. The item under the square root is MSE between $\Tilde{X}$ and $\hat{X}$.
\end{proof}

We further study the relations among the inputs $X$, $Y$, output $\hat{X}$ of the generation model $g$, and the complete time series $\Tilde{X}$ based on their dependencies shown in Figure \ref{fig:graphical_model}.
Firstly, given $\Tilde{X}$ and $\hat{X}$, we show in Lemma \ref{lemma:mse_and_mi} that minimizing their MSE loss is equivalent to maximizing their mutual information $I(\Tilde{X}; \hat{X})$.
Secondly, we prove that adding the retrieved reference $Y$ to the input could reduce the uncertainty for $X$. 
Lemma \ref{lemma_2} shows that methods with $Y$ have a higher lower-bound than those without $Y$ for the mutual information of the ground-truth $\Tilde{X}$ and the predicted values $\hat{X}$: $I(\Tilde{X};\hat{X}) \geq I(\Tilde{X};X, Y) \geq I(\Tilde{X};X)$.
Finally, due to the equivalence of MSE and MI (Lemma \ref{lemma:mse_and_mi}), we can conclude that models with $Y$ (Figure \ref{fig:ref}) are easier to learn than models without $Y$ (Figure \ref{fig:no_ref}) under the MSE loss.

\begin{lemma}[Equivalence between MSE and MI]\label{lemma:mse_and_mi}
Minimizing the MSE loss of $\Tilde{X}$ and $\hat{X}$ is equivalent to maximize the mutual information of $\Tilde{X}$ and $\hat{X}$: $I(\Tilde{X}; \hat{X})$.
\begin{equation}
    \min \mathbb{E}_{p(\Tilde{\mathbf{x}}, \hat{\mathbf{x}})}||\Tilde{\mathbf{x}} - \hat{\mathbf{x}}||^2 \Leftrightarrow \max I(\Tilde{X}, \hat{X})
\end{equation}
where $p(\Tilde{\mathbf{x}}, \hat{\mathbf{x}})$ indicates whether $\Tilde{\mathbf{x}}$ and $\hat{\mathbf{x}}$ is a true pair.
\end{lemma}
\begin{proof}
Minimizing MSE of $\Tilde{\mathbf{x}}$ and $\hat{\mathbf{x}}$ is equivalent to maximizing the log-likelihood $\log p(\Tilde{\mathbf{x}}|\hat{\mathbf{x}})$ \cite{murphy2012machine}, where $p$ is given in Equation \eqref{eq:relation}:
\begin{equation}\label{eq:lemma1_1}
    \min \mathbb{E}_{p(\Tilde{\mathbf{x}}, \hat{\mathbf{x}})}||\Tilde{\mathbf{x}} - \hat{\mathbf{x}}||^2 \Leftrightarrow \max \mathbb{E}_{p(\Tilde{\mathbf{x}}, \hat{\mathbf{x}})}[\log p(\Tilde{\mathbf{x}}|\hat{\mathbf{x}})]
\end{equation}
Therefore, we only need to prove 
\begin{equation}
    \max\mathbb{E}_{p(\Tilde{\mathbf{x}}, \hat{\mathbf{x}})}[\log p(\Tilde{\mathbf{x}}|\hat{\mathbf{x}})]\Leftrightarrow \max I(\Tilde{X};\hat{X})
\end{equation}
In fact,
\begin{equation}
    I(\Tilde{X};\hat{X}) = H(\Tilde{X}) - H(\Tilde{X}|\hat{X})
\end{equation}
and $H(\Tilde{X})$ is a constant since the ground-truth $\Tilde{X}$ is fixed in the dataset.
Thus, we have
\begin{equation}\label{eq:lemma1_2}
    \max  I(\Tilde{X};\hat{X}) \Leftrightarrow \max - H(\Tilde{X}|\hat{X})
\end{equation}
According to the definition of conditional entropy, we have 
\begin{equation}\label{eq:lemma1_3}
    - H(\Tilde{X}|\hat{X}) = \mathbb{E}_{p(\Tilde{\mathbf{x}}, \hat{\mathbf{x}})}[\log p(\Tilde{\mathbf{x}}|\hat{\mathbf{x}})]
\end{equation}
The proof is concluded by combining Equations \eqref{eq:lemma1_1}\eqref{eq:lemma1_2}\eqref{eq:lemma1_3}.
\end{proof}

\begin{lemma}[MI Monotonicity]\label{lemma_2}
The following MI inequalities hold for the graphical model shown in Figure \ref{fig:ref}.
\begin{equation} 
    I(\Tilde{X};\hat{X}) \geq I(\Tilde{X};X, Y) \geq I(\Tilde{X};X)
\end{equation}
\end{lemma}
\begin{proof}[Sketch of Proof]
The first inequality is derived based on the data processing inequality \cite{mezard2009information}.
The second inequality holds since $I(\Tilde{X}; X, Y)=I(\Tilde{X};X)+I(\Tilde{X};Y|X)$ and $I(\Tilde{X};Y|X) \geq 0$.
\end{proof}

\section{Methodology}\label{sec:method}
Figure \ref{fig:overview} is an overview of \retime.
In the first stage, \retime\ retrieves the top $K$ references $\{\mathbf{Y}_k\}_{k=1}^K$ for the target $\mathbf{X}$ from the database $\{\mathbf{Y}_n'\}_{n=1}^N$, based on the relations $\mathcal{R}$ between the target and database and the relation graph $\mathbf{A}'$ of the database.
In the second stage, \retime\ combines $\mathbf{X}$ and $\{\mathbf{Y}_k\}_{k=1}^K$ to generate $\hat{\mathbf{X}}$.

\subsection{Relational Retrieval}\label{sub_sec:reteival}
When the missing rate of the target $\mathbf{X}$ is high (e.g., 0.8), it will be hard to accurately complete $\mathbf{X}$ merely based on the observed historical content of $\mathbf{X}$.
Even for recent forecasting methods \cite{zhou2021informer}, the uncertainty of the prediction accuracy could be very high under a high missing rate.
To reduce the uncertainty, we propose to retrieve $K$ references $\{\mathbf{Y}_k\}_{k=1}^K$ from the database $\{\mathbf{Y}_n'\}_{n=1}^N$, based on the relations $\mathcal{R}$ and $\mathbf{A}'$.
We choose relational retrieval over content-based retrieval since the observed historical content could be noisy and unreliable when the missing rate is high.

Given $\mathbf{A}'\in\mathbb{R}^{N\times N}$ and $\mathcal{R}$, we first construct a new adjacency matrix $\mathbf{A}\in\mathbb{R}^{(N+1)\times (N+1)}$ by appending $\mathcal{R}$ to the last row and column of $\mathbf{A}'$.
Then we use the Random Walk with Restart (RWR) \cite{tong2006fast} to obtain the relational proximity scores between $\mathbf{X}$ and $\{\mathbf{Y}_n'\in\mathbb{R}^{T'\times v}\}_{n=1}^N$, whose closed-form solution is given by:
\begin{equation}\label{eq:rwr}
    \mathbf{p} = (1-c)(\mathbf{I} - c\Tilde{\mathbf{A}})^{-1}\mathbf{e}
\end{equation}
where $\Tilde{\mathbf{A}}$ is the normalized adjacency matrix, $\mathbf{I}$ is the identity matrix and $c\in(0, 1)$ is the tunable damping factor; 
$\mathbf{e}\in\{0,1\}^{N+1}$ is the indicator vector of $\mathbf{X}$,  where $\mathbf{e}[N+1]=1$ and $\mathbf{e}[n]=0$ for $\forall n\in[1, \cdots, N]$;
$\mathbf{p}\in\mathbb{R}^{N+1}$ is the relational proximity vector where $\mathbf{p}[n]$ describes the proximity between $\mathbf{X}$ and $\mathbf{Y}_n'$. 
Given $\mathbf{p}$, we retrieve the top $K$ time series $\{\mathbf{Y}_k\}_{k=1}^K$ as the references.

Instead of using the entire $\mathbf{Y}_k\in\mathbb{R}^{T'\times v}$ as the reference, we only use a $T$-length ($T\ll T'$) snippet of it, since using the entire sequence could introduce much irrelevant noisy information.
Let $T_S$ and $T_S'$ be the start time of the target $\mathbf{X}$ and the reference snippet respectively.
Many time series have clear periodicity patterns, such as the weekly pattern of traffic monitoring data and the yearly pattern of the air temperature.
In this paper, we exploit the periodicity patterns and set the time difference between the target and $\Delta T = T_S - T_S'$ as the length of one period.

\subsection{Content Synthesis}\label{sub_sec:generation}
Despite the relational closeness of $\mathbf{X}$ and $\{\mathbf{Y}_k\}_{k=1}^K$, they usually have content discrepancies. 
For example, peaks and valleys of $\mathbf{X}$ and $\mathbf{Y}_k$ might be different, and $\mathbf{Y}_k$ might be noisy.
Besides, there are also temporal dependencies among different time steps.
Therefore, it is necessary to build a model to combine their content.

Figure \ref{fig:content_synthesis} presents an illustration of the content synthesis model, which is comprised of the input, aggregation, and output modules.
The input module maps $\mathbf{X}\in\mathbb{R}^{T\times v}$ and $\{\mathbf{Y}_k\in\mathbb{R}^{T\times v}\}_{k=1}^K$ into embeddings $\mathbf{H}\in\mathbb{R}^{(K+1)\times T \times d}$, where $d$ is the size of hidden dimension.
Then the aggregation module aggregates the content of $\mathbf{H}$ across the $K+1$ time series and $T$ time steps into the aggregated embeddings $\mathbf{H}$.
Finally, the output module generates the completed time series $\hat{\mathbf{X}}$ based on the aggregated embeddings $\mathbf{H}$.

\begin{figure}
    \centering
    \includegraphics[width=.23\textwidth]{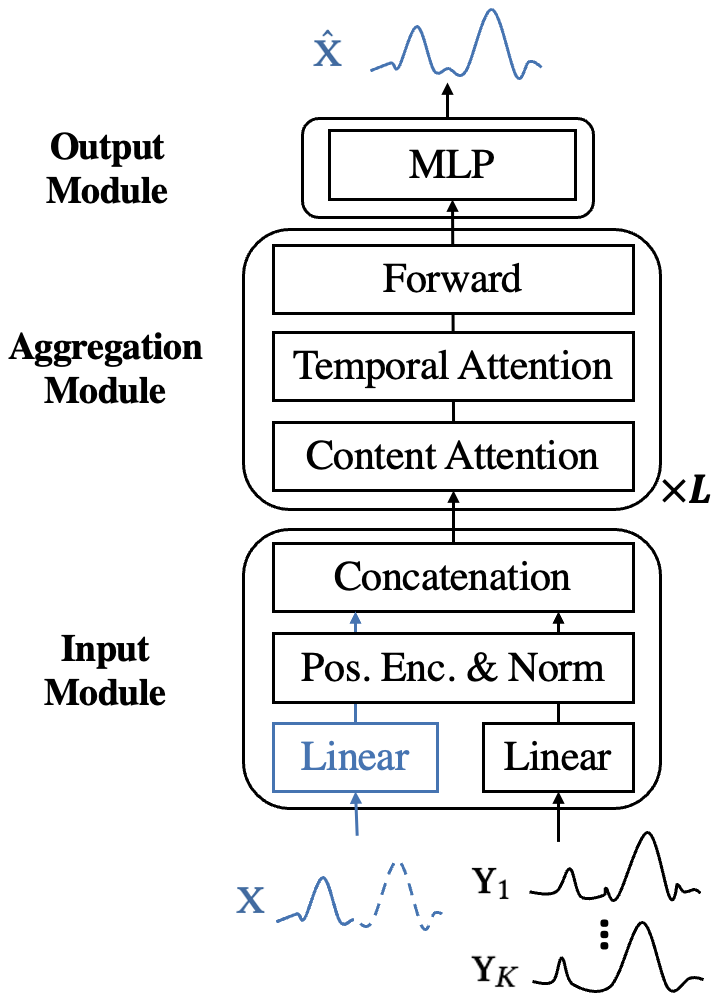}
    \caption{Content Synthesis}
    \label{fig:content_synthesis}
\end{figure}

\noindent{\bf A - Input Module.}
The input module maps the time series into the embedding space.
Given $\mathbf{X}$ and $\{\mathbf{Y}_k\}_{k=1}^K$, we first apply separate linear layers to them.
Then we apply the position encoding \cite{vaswani2017attention} and layer norm \cite{ba2016layer} to obtain the embeddings $\mathbf{H}\in\mathbb{R}^{T\times d}$ and $\{\mathbf{H}'_k\in\mathbb{R}^{T\times d}\}_{k=1}^K$.
Finally, they are concatenated into $\hat{\mathbf{H}}\in\mathbb{R}^{(K+1)\times T\times d}$, where $\hat{\mathbf{H}}[1:K]=[\mathbf{H}'_1,\dots, \mathbf{H}_K']$ and $\hat{\mathbf{H}}[K+1]={\mathbf{H}}$.

\noindent{\bf B - Aggregation Module.}
The aggregation module jointly considers the content discrepancies between $\mathbf{H}$ and $\{\mathbf{H}'_k\}_{k=1}^K$ at each time step and the temporal dependencies across time steps.
Based on the multi-head self-attention \cite{vaswani2017attention}, we build content and temporal attention models to handle the content discrepancies and temporal dependencies.
The aggregation module sequentially applies the content and temporal attention models to $\hat{\mathbf{H}}\in\mathbb{R}^{(K+1)\times T\times d}$, which is comprised of $L$ blocks, and the structure of the block is shown in Figure \ref{fig:content_synthesis}.
Within the block, firstly, the content attention computes content attention scores over the first dimension ($K+1$ time series) of $\hat{\mathbf{H}}$ for each step $t\in[1,\cdots, T]$, and produces the content embeddings $\mathbf{Z}\in\mathbb{R}^{(K+1)\times T\times d}$.
Secondly, the temporal attention computes temporal attention scores over the second dimension ($T$ steps) of $\mathbf{Z}$ for each time series $k\in[1, \cdots, K+1]$, and encodes temporal information into $\mathbf{Z}$.
Finally, a forward layer \cite{vaswani2017attention} is applied to $\mathbf{Z}$ to obtain the aggregated embeddings $\hat{\mathbf{H}}\in\mathbb{R}^{(K+1)\times T\times d}$.

\noindent{\bf C - Output Module.}
The output module maps the aggregated embeddings $\hat{\mathbf{H}}$ from the embedding space $\mathbb{R}^d$ back into the original space $\mathbb{R}^V$ of the time series.
Specifically, the output module takes input as $\hat{\mathbf{H}}[K+1]$, which is the embeddings of the target $\mathbf{X}$.
Then a Multi-Layer Perceptron (MLP) is applied over $\mathbf{H}[K+1]$ to generate the predicted target time series $\hat{\mathbf{X}}$.

\subsection{Adaptation to Other Settings}\label{sub_sec:co-evolv}
\noindent\textbf{A - Spatial-Temporal Time Series.}
Spatial-temporal time series is ubiquitous and have attracted a lot of attention~\cite{yu2016temporal, NEURIPS2020_cdf6581c, wu2020connecting, jing2021network}.
Different from the retrieval-based single time series forecasting, where the target is a single time series \emph{out of} the database $\mathbf{X}\notin\{\mathbf{Y}_n\}_{n=1}^N$,
the target of the spatial-temporal time series is the entire dataset $\mathbf{X}=\{\mathbf{Y}_n\}_{n=1}^N$.
\retime\ can be naturally adapted to spatial-temporal time series by treating each single time series of spatial-temporal time series as the target and the rest time series as the database.
The relational retrieval can be interpreted as the diffusion graph kernel \cite{tong2006fast}, which identifies the most important neighbors, and the content synthesis aggregates the information of the neighbors.

\noindent\textbf{B - Time Series Imputation.}
It is obvious that \retime\ can be naturally applied to the time series imputation task for regularly-sampled time series.

\subsection{Training}
During training, given a complete target time series $\Tilde{\mathbf{X}}$, we generate a binary mask $\mathbf{M}$ to obtain the incomplete target time series $\mathbf{X} = \mathbf{M}\odot\Tilde{\mathbf{X}}$, where $\odot$ is the Hadamard product.
For forecasting, the values after the pre-defined time step $\tau$ are set as zeros.
For imputation, we randomly generate the values for the masks according to the pre-defined missing rate.
Then we feed $\mathbf{X}$ with its references $\{\mathbf{Y}_k\}_{k=1}^K$, which are retrieved by the relational retrieval model, to the content synthesis model to obtain $\hat{\mathbf{X}}$.
We use the standard Mean Squared Error (MSE) between $\hat{\mathbf{X}}$ and $\Tilde{\mathbf{X}}$ to train the model.

\begin{figure*}[t]
\centering
\begin{subfigure}[b]{.23\textwidth}
  \includegraphics[width=\linewidth]{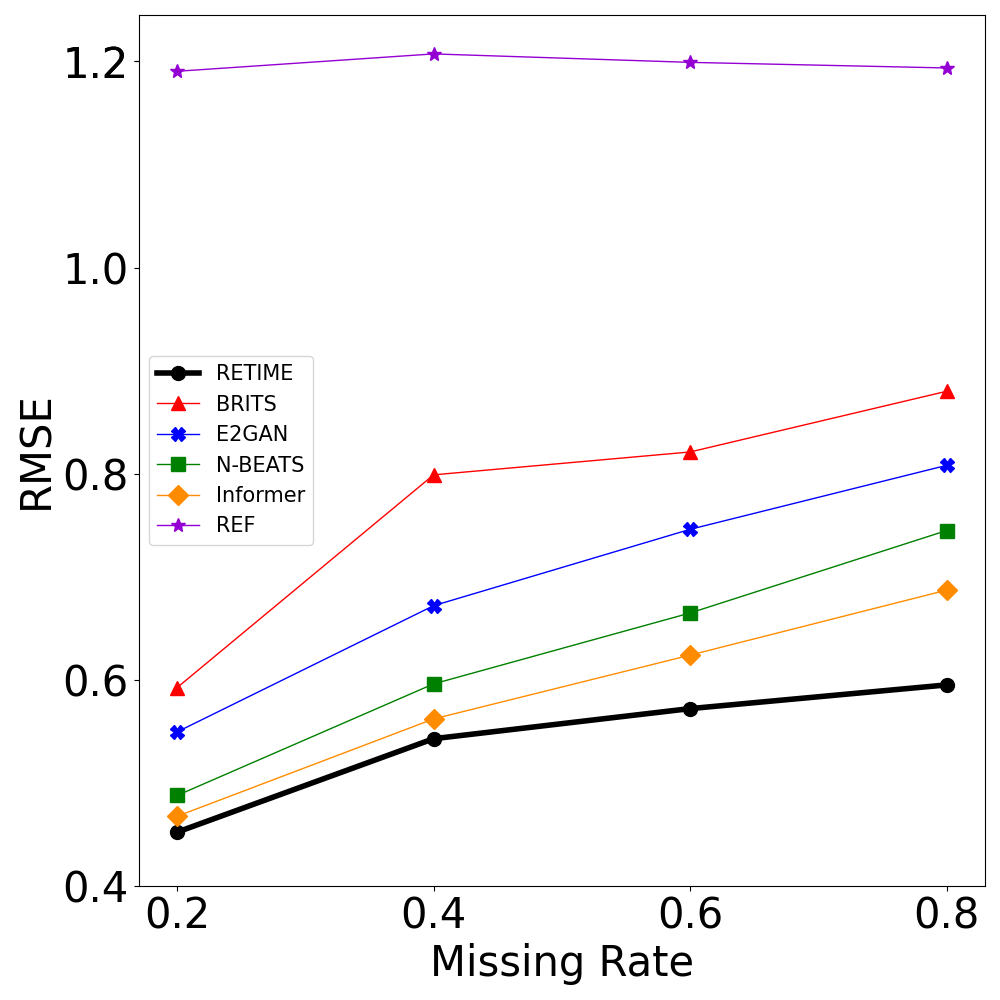}\label{fig:forecast_traffic_single}
  \caption{Traffic Forecasting}
\end{subfigure}\,
\begin{subfigure}[b]{.23\textwidth}
  \includegraphics[width=\linewidth]{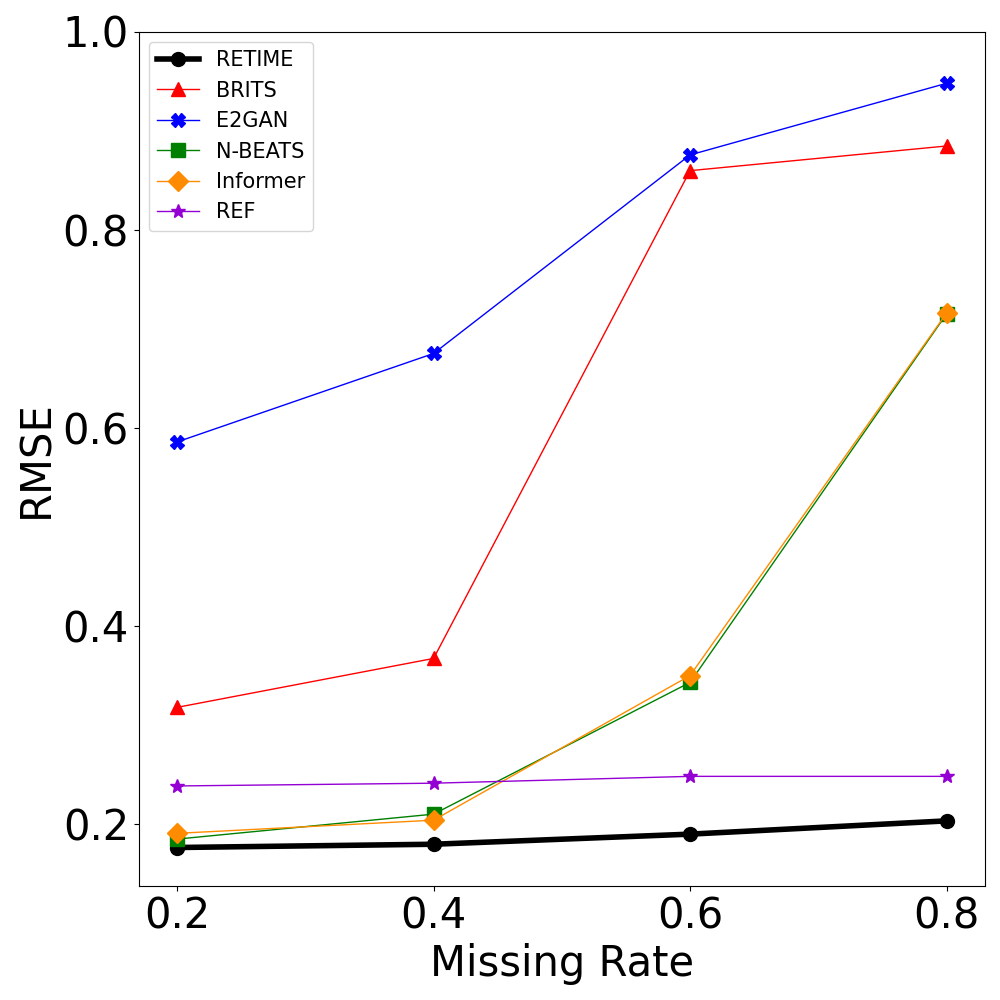}\label{fig:forecast_noaa_single}
  \caption{Temperature Forecasting}
\end{subfigure}\,
\begin{subfigure}[b]{.23\textwidth}
 \includegraphics[width=\linewidth]{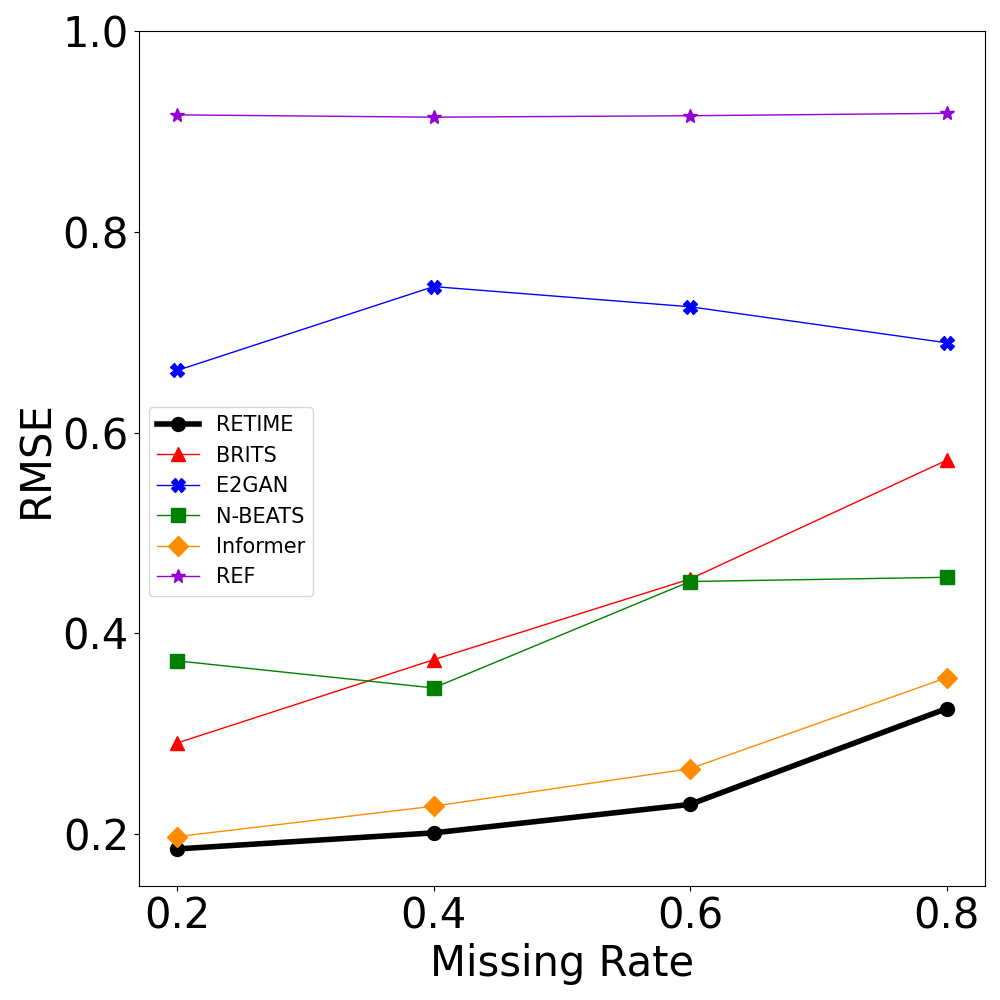}\label{fig:impute_traffic_single}
 \caption{Traffic Imputation}
\end{subfigure}\,
\begin{subfigure}[b]{.23\textwidth}
  \includegraphics[width=\linewidth]{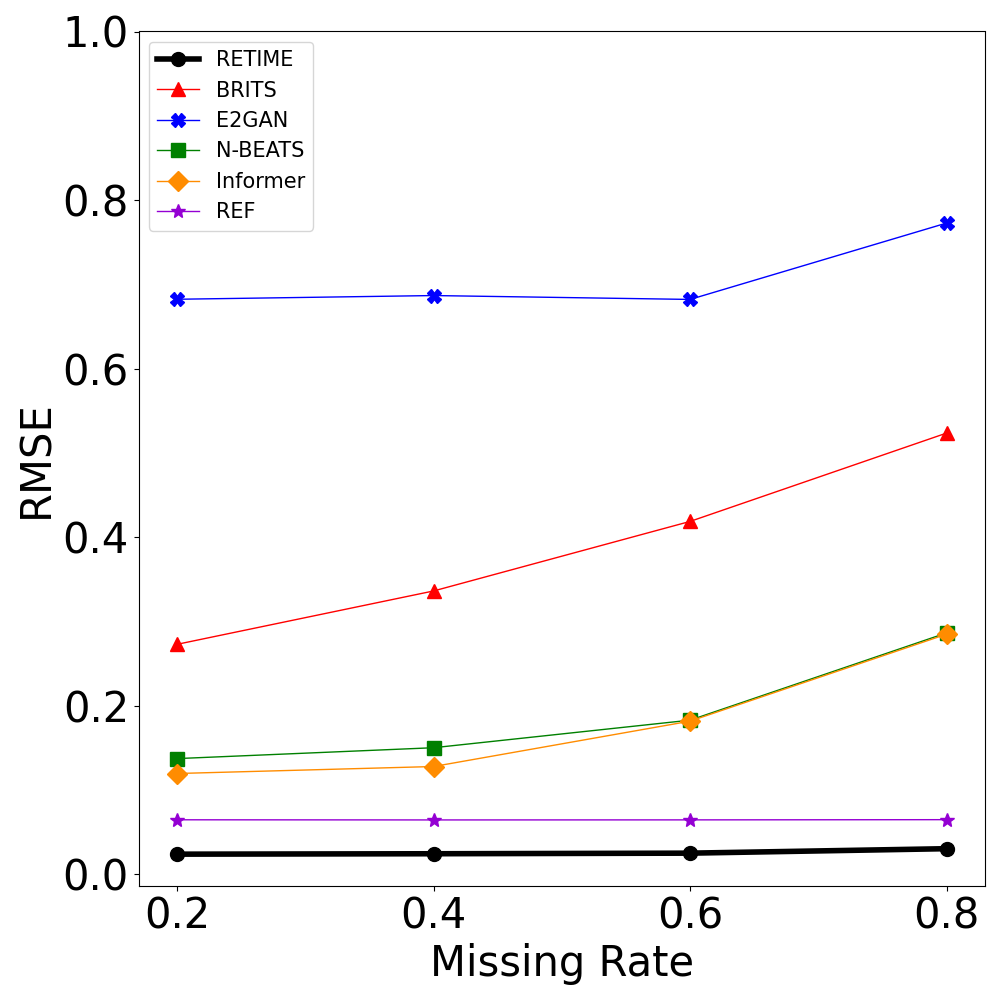}\label{fig:impute_noaa_single}
  \caption{Temperature Imputation}
\end{subfigure}\,
\caption{RMSE scores on single time series forecasting and imputation. The lower the better.}\label{fig:main_single}
\end{figure*}

\begin{figure*}[t]
\centering
\begin{subfigure}[b]{.23\textwidth}
  \includegraphics[width=\linewidth]{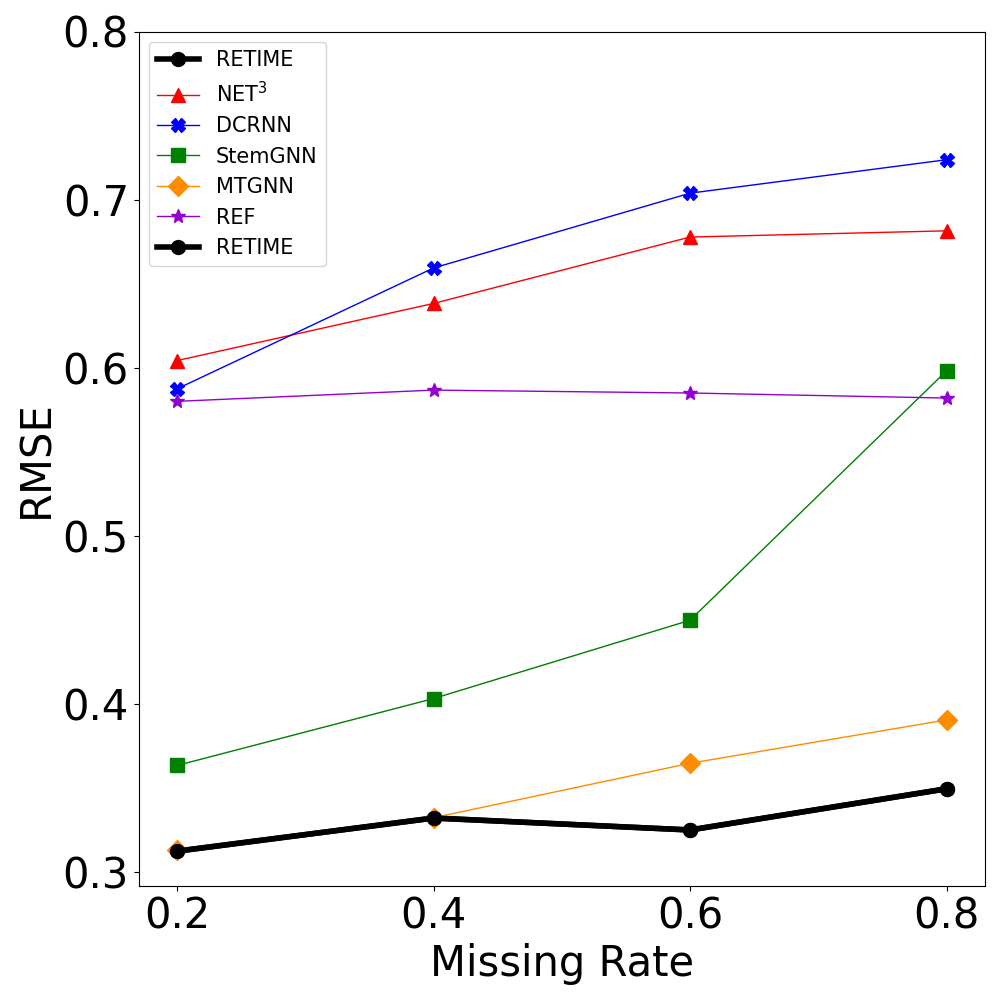}\label{fig:forecast_traffic_co}
  \caption{Traffic Forecasting}
\end{subfigure}\,
\begin{subfigure}[b]{.23\textwidth}
  \includegraphics[width=\linewidth]{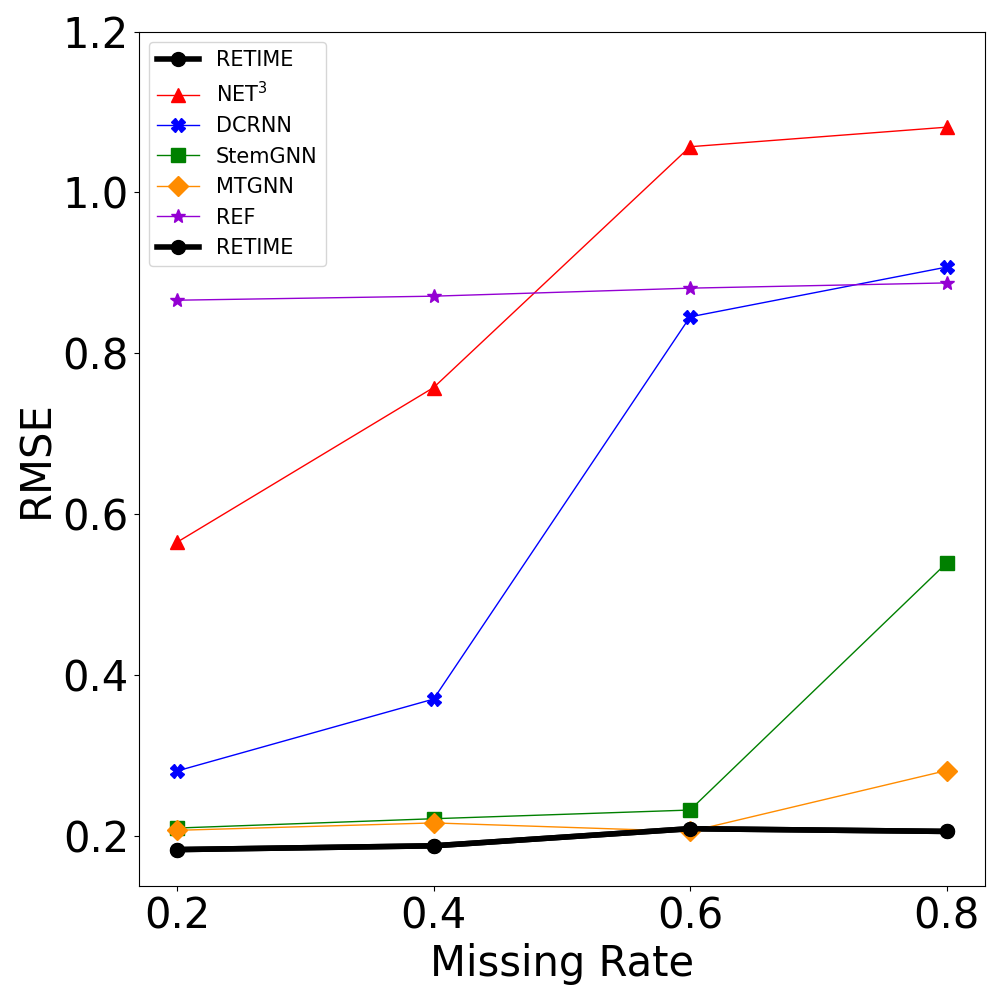}\label{fig:forecast_noaa_co}
  \caption{Temperature Forecasting}
\end{subfigure}\,
\begin{subfigure}[b]{.23\textwidth}
 \includegraphics[width=\linewidth]{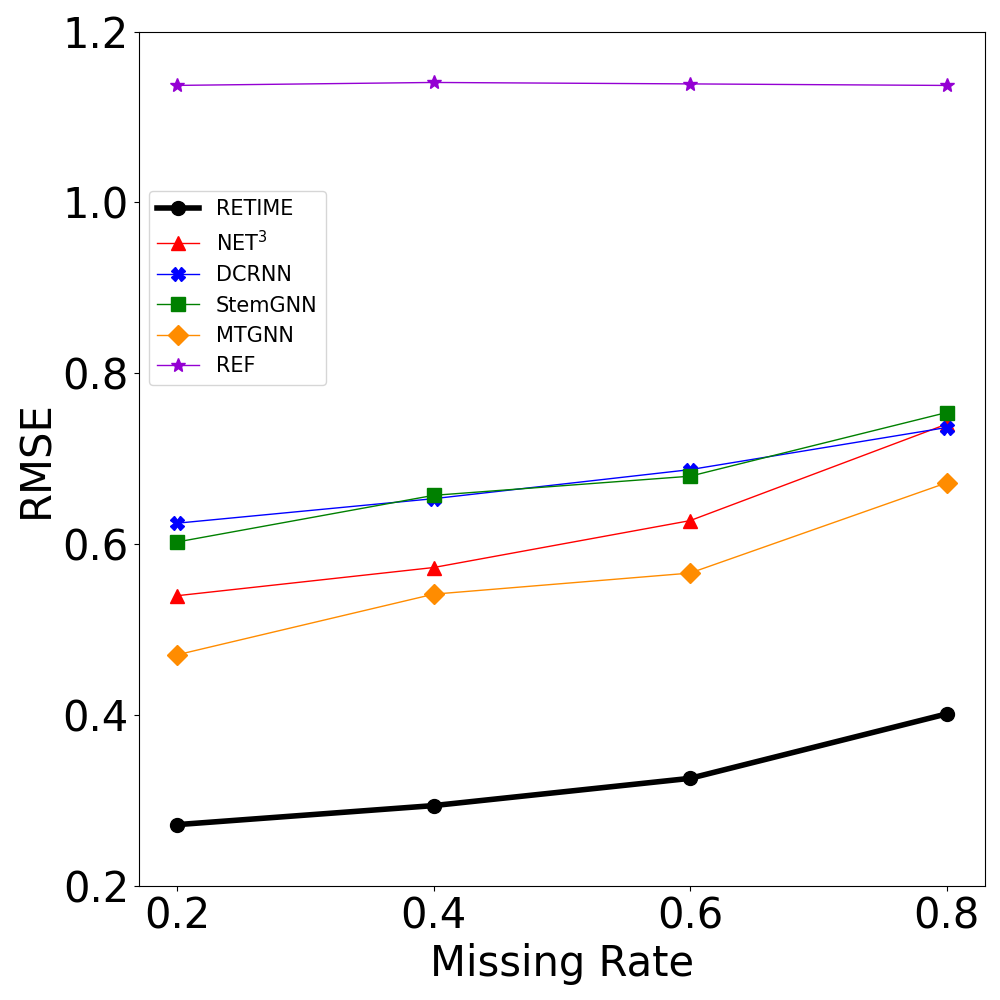}\label{fig:impute_traffic_co}
 \caption{Traffic Imputation}
\end{subfigure}\,
\begin{subfigure}[b]{.23\textwidth}
  \includegraphics[width=\linewidth]{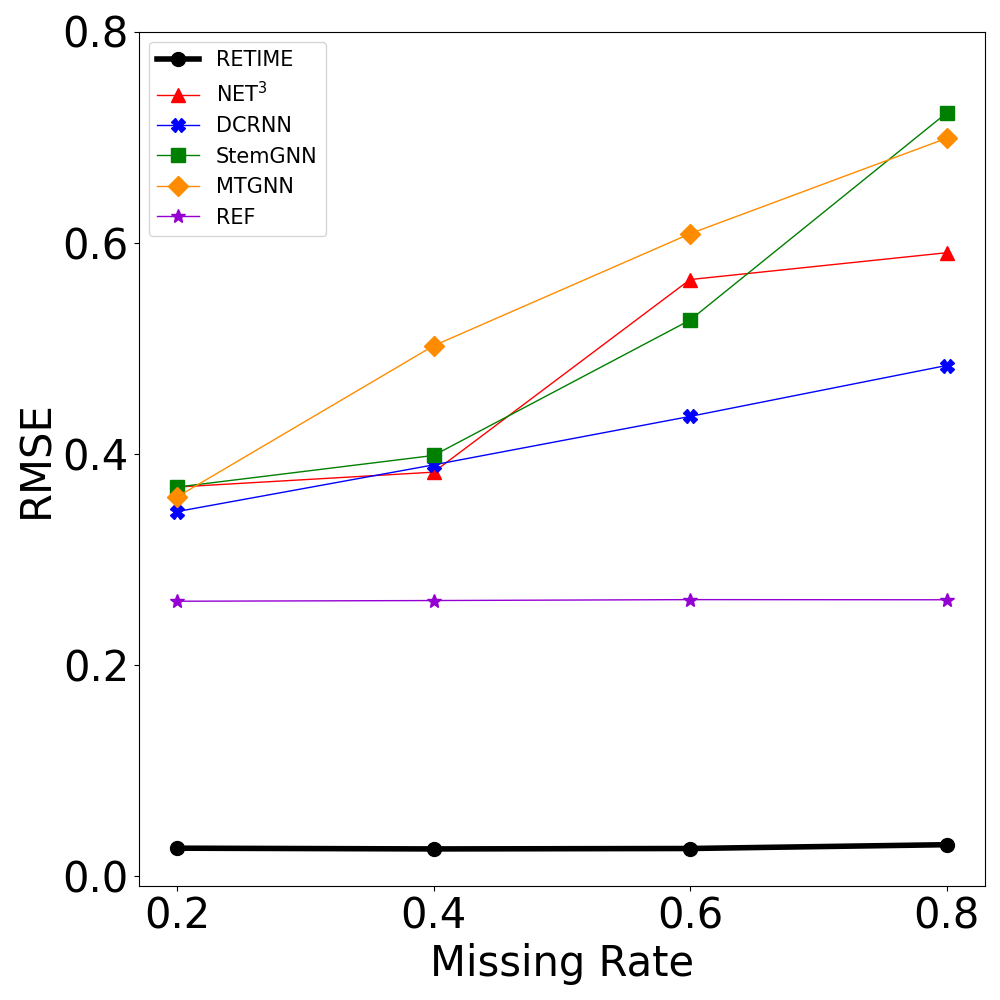}\label{fig:impute_noaa_co}
  \caption{Temperature Imputation}
\end{subfigure}\,
\caption{RMSE scores on spatial-temporal time series forecasting and imputation. The lower the better.}\label{fig:main_co}
\end{figure*}

\section{Experiments}\label{sec:experiments}

\subsection{Experimental Setup}

\noindent\textbf{Datasets.}
We evaluate \retime\ on two real-world datasets, where the shapes of the datasets are formulated as $N \times T \times v$:

\textit{Traffic} dataset is collected from Caltrans Performance Measurement System (PeMS)\footnote{\url{https://dot.ca.gov/programs/traffic-operations/mpr/pems-source}}.
It contains hourly average speed and occupancy collected from 2,000 sensor stations in District 7 of California during June 2018.
The size of the dataset is $2000\times720\times2$.
The relation between two stations is whether they are adjacent.

\textit{Temperature} is a subset of version 3 of the 20th Century Reanalysis\footnote{\url{https://psl.noaa.gov/data/gridded/data.20thC_ReanV3.html}} data \cite{slivinski2019towards}.
It contains the monthly average temperature from 2001 to 2015, which covers a $30\times30$ area of North America, from $30^\circ$ N to $60^\circ$ N, $80^\circ$ W to $110^\circ$ W.
The shape is $900\times180\times1$.
The relation between two locations is whether they are adjacent.

For the single time series setting, given $N$ time series, we randomly select 10\%/10\% time series for validation and testing.
As a result, the train/validation/test splits are 1600/200/200 and 720/90/90 for the traffic and temperature datasets.
For the spatial-temporal setting, we select 10\%/10\% \textit{snippets} from the entire datasets for validation/test. 
After splitting data, we segment each time series into 24/12-length snippets for traffic/temperature datasets respectively, which cover one day/one year.
Time series are normalized according to the mean and standard deviation of the training sets.

\noindent\textbf{Evaluation Tasks.}
We evaluate \retime\ for both single time series and spatial-temporal time series, and on both forecasting and imputation tasks.
For each setting, we fix the snippet length of the target and change the missing rate from 0.2 to 0.8.

\noindent\textbf{Comparison Methods.}
We compare \retime\ with the following baselines. 
The methods for single time series setting include block-style methods: N-BEATS \cite{oreshkin2019n} and Informer \cite{zhou2021informer}, and RNN based methods: BRITS \cite{cao2018brits}, E2GAN \cite{luo2019e2gan}.
The methods for spatial-temporal time series setting include block-style methods: StemGNN \cite{NEURIPS2020_cdf6581c} and MTGNN \cite{wu2020connecting}, and RNN based methods: DCRNN \cite{li2017diffusion}, NET$^3$ \cite{jing2021network}.
Additionally, we also use the reference (REF) with the highest relational score to the target as a baseline. 

\noindent\textbf{Implementation Details.}
The hidden dimensions for the traffic/temperature datasets are 256/64. 
The numbers of blocks and attention heads are 8/4 for the traffic/temperature datasets.
The learning rates are tuned within [0.001, 0.0001].
Early stopping is applied on the validation set to prevent over-fitting.
$K$ is tuned within [1, 5, 10, 20].
For the relational retrieval, we set the damping factor of RWR as $c=0.9$.
Given a target time series snippet, which starts from the time $T_S$, we select reference snippets starting from $T_S'$, and denote their difference as $\Delta T=T_S-T_S'$.
For forecasting, $\Delta T$ is one week/one year for the traffic/temperature datasets, which is generally the length of one cycle.
For imputation, we set $\Delta T = 0$, which has the best performance in experiments.
When applying forecasting models to imputation tasks, we take the entire $\mathbf{X}$ as input and force the models to predict the entire $\Tilde{\mathbf{X}}$.

\begin{figure*}[t]
\centering
\begin{subfigure}[b]{.23\textwidth}
  \includegraphics[width=\linewidth]{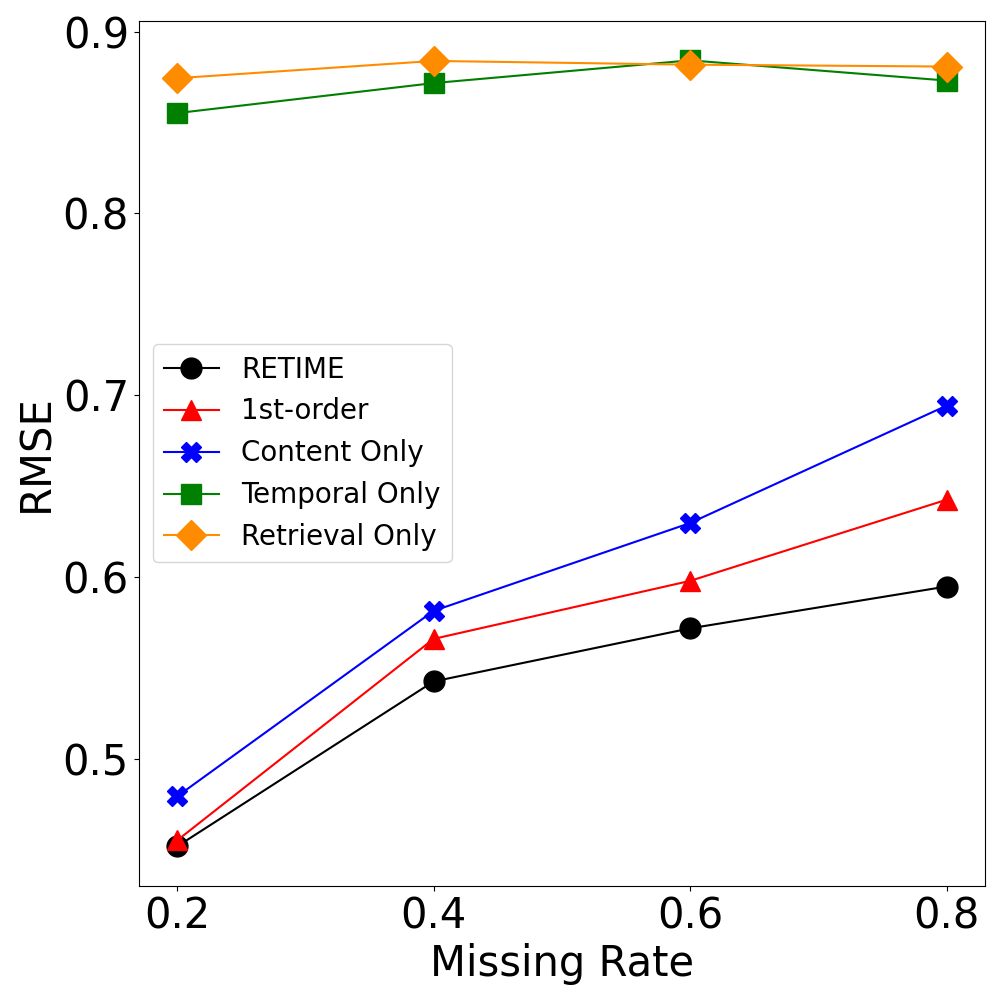}\label{fig:ablation_single_forecast}
  \caption{Single Forecasting}
\end{subfigure}\,
\begin{subfigure}[b]{.23\textwidth}
\includegraphics[width=\linewidth]{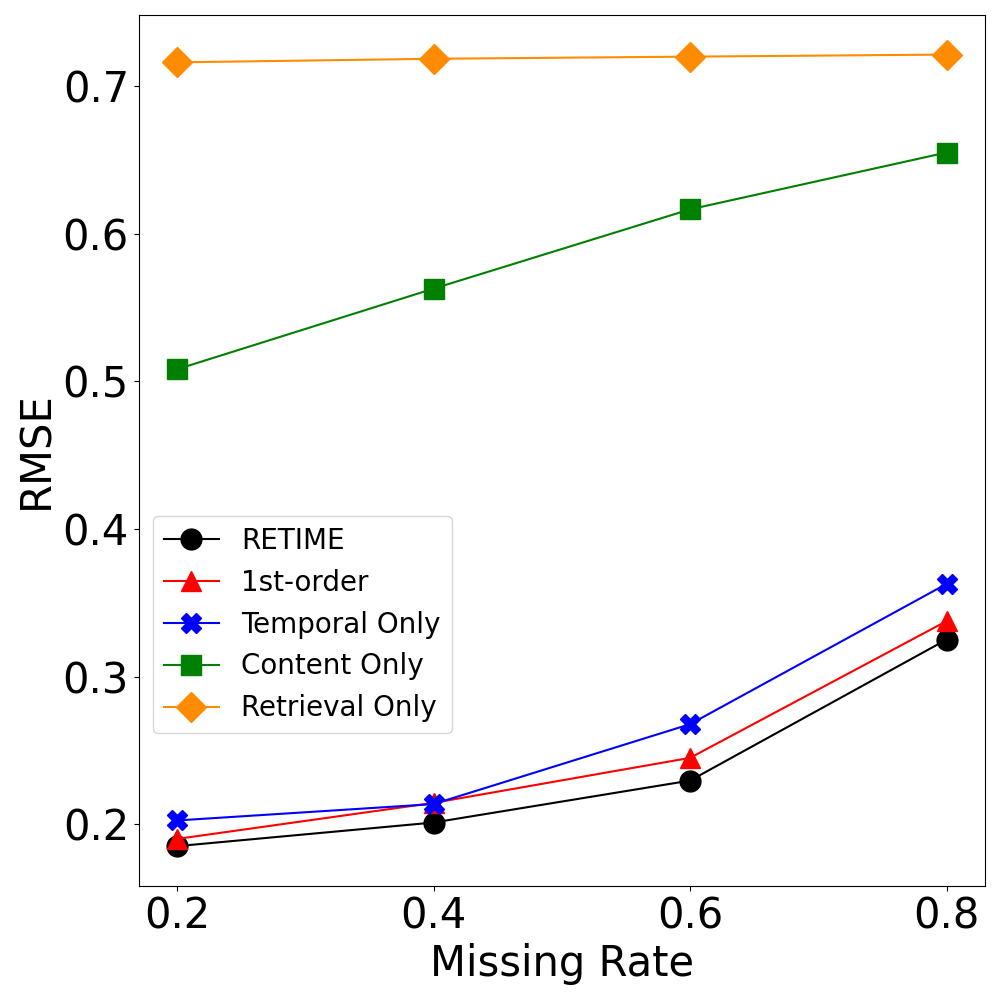}\label{fig:ablation_single_impute}
 \caption{Single Imputation}
\end{subfigure}\,
\begin{subfigure}[b]{.23\textwidth}
  \includegraphics[width=\linewidth]{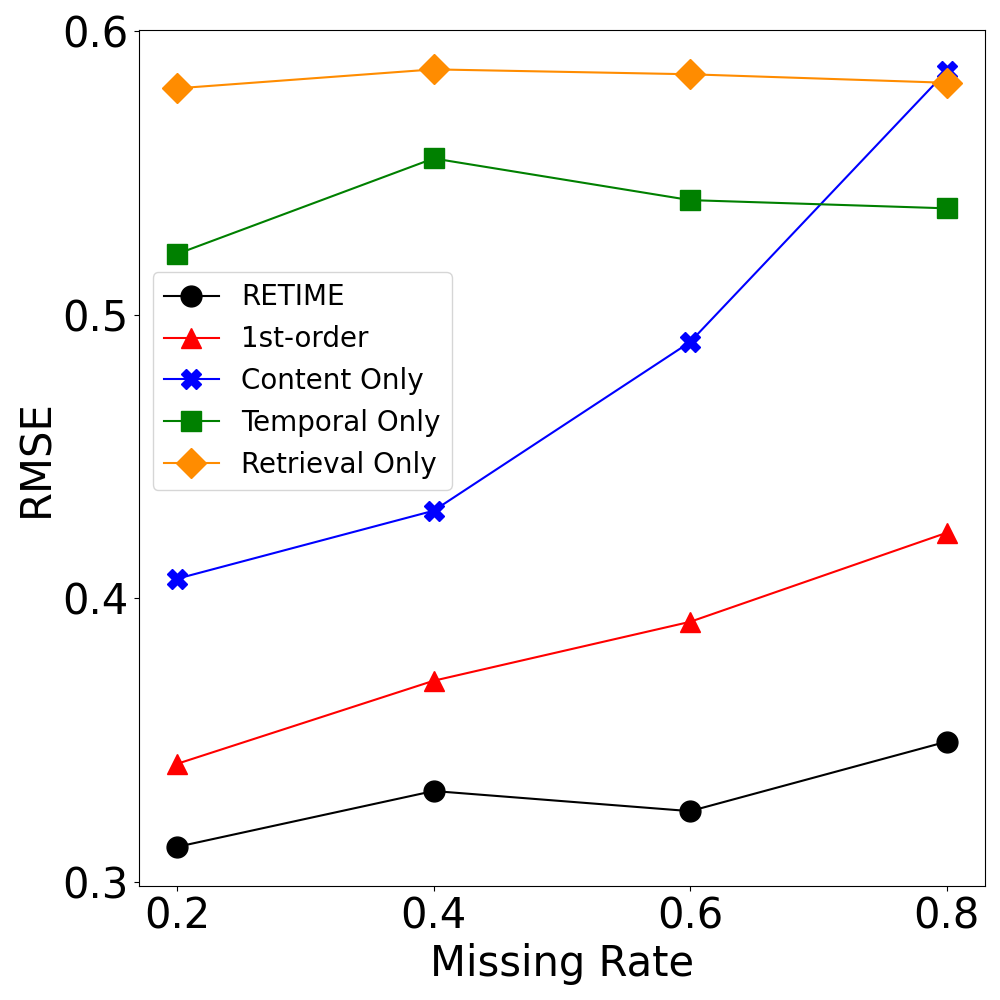}\label{fig:ablation_co_forecast}
  \caption{Spatial-Temporal Forecasting}
\end{subfigure}\,
\begin{subfigure}[b]{.23\textwidth}
  \includegraphics[width=\linewidth]{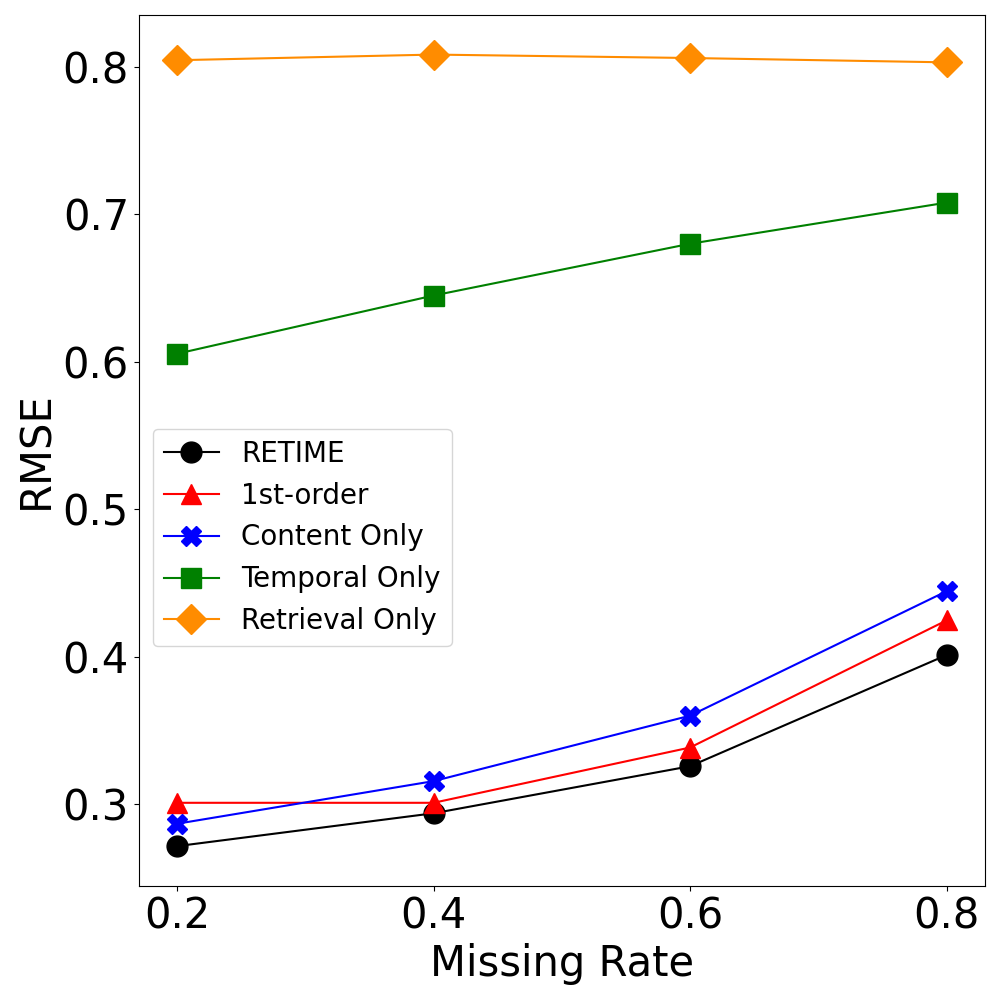}\label{fig:ablation_co_impute}
  \caption{Spatial-Temporal Imputation}
\end{subfigure}\,
\caption{Ablation study on the traffic dataset.}
\label{fig:ablation}
\end{figure*}

\subsection{Main Results}
We compare \retime\ with various baselines for single/spatial-temporal time series forecasting/imputation tasks for different missing rates.

\noindent\textbf{Single Time Series.}
The results for single time series are presented in Figure \ref{fig:main_single}, where the first and second rows show the Root Mean Squared Error (RMSE) for forecasting and imputation respectively.
For both forecasting and imputation, compared with RNN based methods, i.e., BRITS and E2GAN, recent block-style methods, i.e., N-BEATS and Informer, have lower RMSE scores. 
\retime\ has much lower RMSE scores than N-BEATS and Informer, demonstrating the effectiveness of the proposed strategy.

An interesting observation for the temperature dataset is that the simple retrieval baseline REF performs significantly better than other state-of-the-art baselines, corroborating the power of the reference time series.
To be more specific, firstly, as indicated by the performance of the state-of-the-art methods, there might exist some complex temporal patterns for the temperature data, which could not be easily captured by models merely based on the observed content of the targets.
Secondly, the superior performance of REF over the state-of-the-art methods demonstrates that the retrieved time series can indeed significantly reduce uncertainty.

\begin{figure}[t]
\centering
\begin{subfigure}[b]{.23\textwidth}
\includegraphics[width=\linewidth]{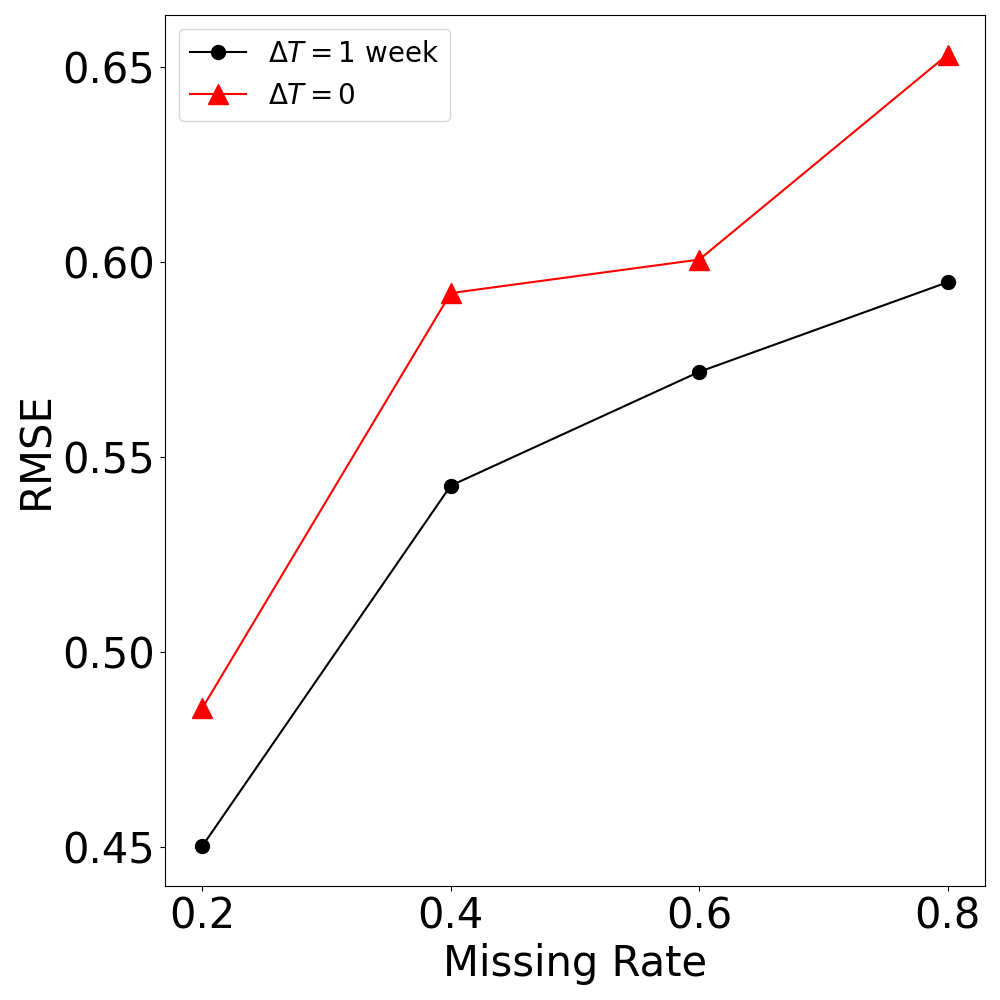}\label{fig:time_shift_single}
 \caption{Single Forecasting}
\end{subfigure}\,
\begin{subfigure}[b]{.23\textwidth}
  \includegraphics[width=\linewidth]{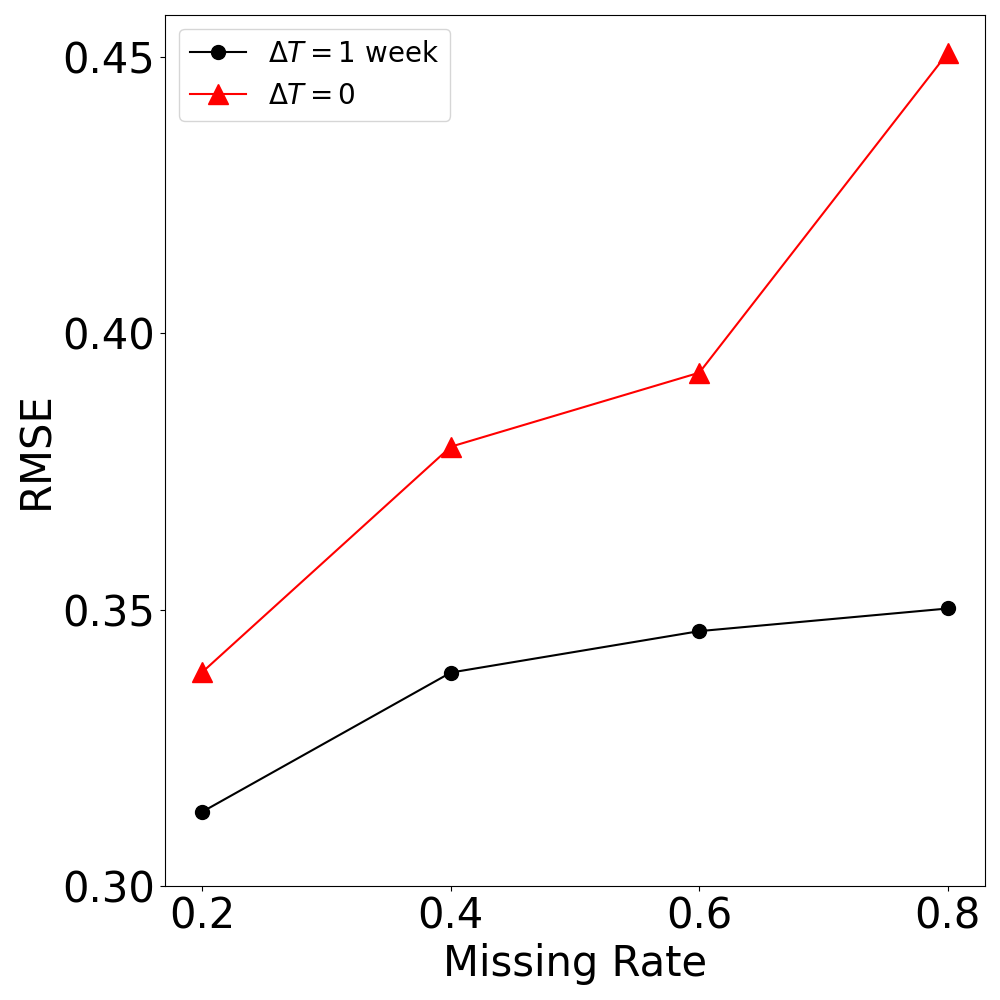}\label{fig:time_shift_co}
  \caption{Spatial-Temporal Forecasting}
\end{subfigure}
\caption{Effect of using prior snippets for forecasting.}
\label{fig:ablation_prior}
\end{figure}

\noindent\textbf{Spatial-Temporal Time Series.}
The results for spatial-temporal time series are presented in Figure \ref{fig:main_co}.
The general observation is similar to the single time series that \retime\ consistently performs better than the state-of-the-art methods.

\subsection{Ablation Study}
We study the impact of each component of \retime\ on the traffic dataset, which is the largest dataset.

\noindent {\bf Impact of the Attentions.}
As shown in Figure \ref{fig:ablation}, compared with the full model \retime\ , if we only use either the content or temporal attention, the performance will drop.
Besides, spatial attention alone performs worse than temporal attention alone, showing the importance of modeling temporal dependencies.

\noindent {\bf Beyond the First Order Neighbors.}
In Figure \ref{fig:ablation}, the ``1st-order'' is \retime\  using the 1st-order neighbors of the targets as references.
\retime\ is better than ``1st-order'', indicating that it is important to choose appropriate neighbors as references.
RWR used in the relational retrieval stage could capture global proximity scores.

\noindent {\bf Effect of the Content Synthesis Model.}
The ``Retrieval Only'' in Figure \ref{fig:ablation} uses the average of the $K$ retrieved references as the prediction.
Compared with the retrieval-only model, \retime\ performs better, demonstrating the effectiveness of the content synthesis.

\noindent {\bf Effect of Using Prior Snippets for Forecasting.}
When the time series has clear periodic patterns, it is natural to resort to the historical snippets in the database for help.
In Figure \ref{fig:ablation_prior}, ``$\Delta T=0$'' means that the target and reference snippets have the same start time, and thus the values of both the targets and references are zeros after the separation time $\tau$.
``$\Delta T=1$ week'' means that the start time of the reference snippets is one week before the targets. 
Figure \ref{fig:ablation_prior} shows that using prior snippets can significantly improve the model's performance.

\begin{figure*}[t]
\centering
\begin{subfigure}[b]{.23\textwidth}
  \includegraphics[width=\linewidth]{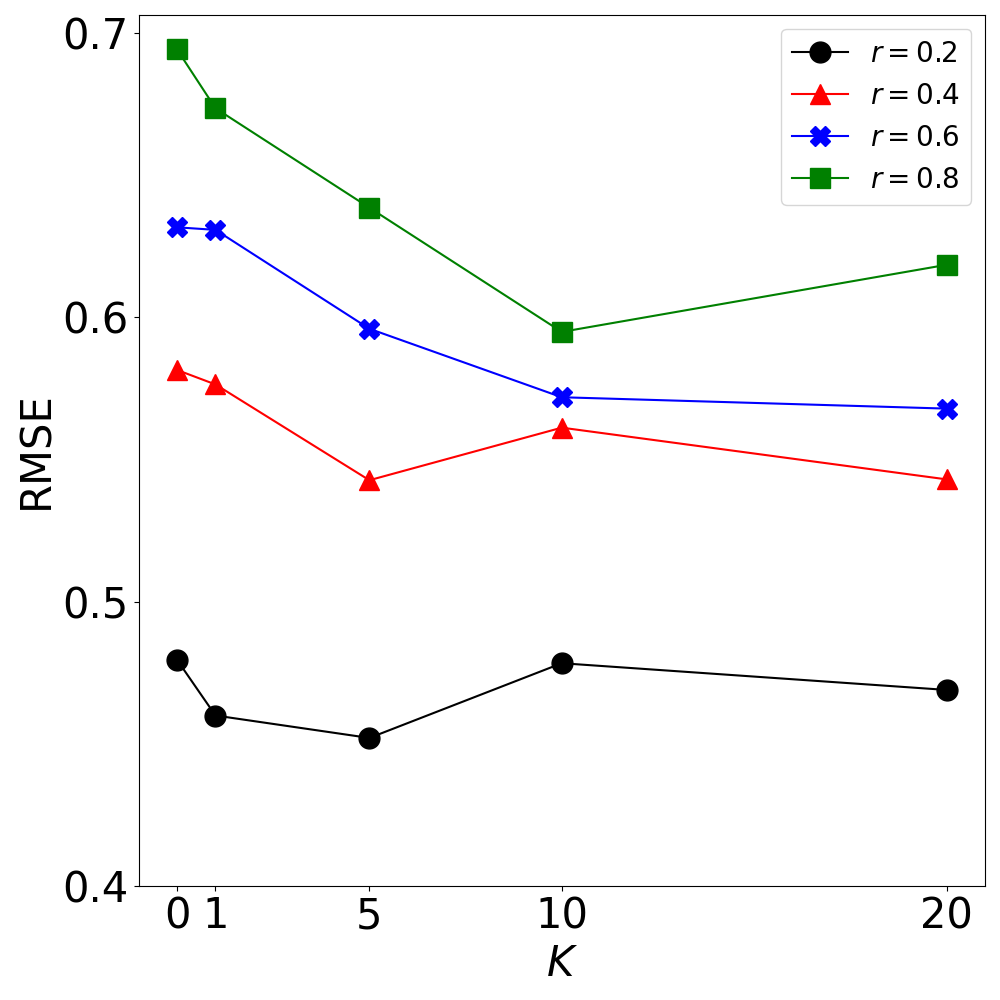}
  \caption{Single Forecasting}\label{fig:sensitivity_single_forecast}
\end{subfigure}\,
\begin{subfigure}[b]{.23\textwidth}
\includegraphics[width=\linewidth]{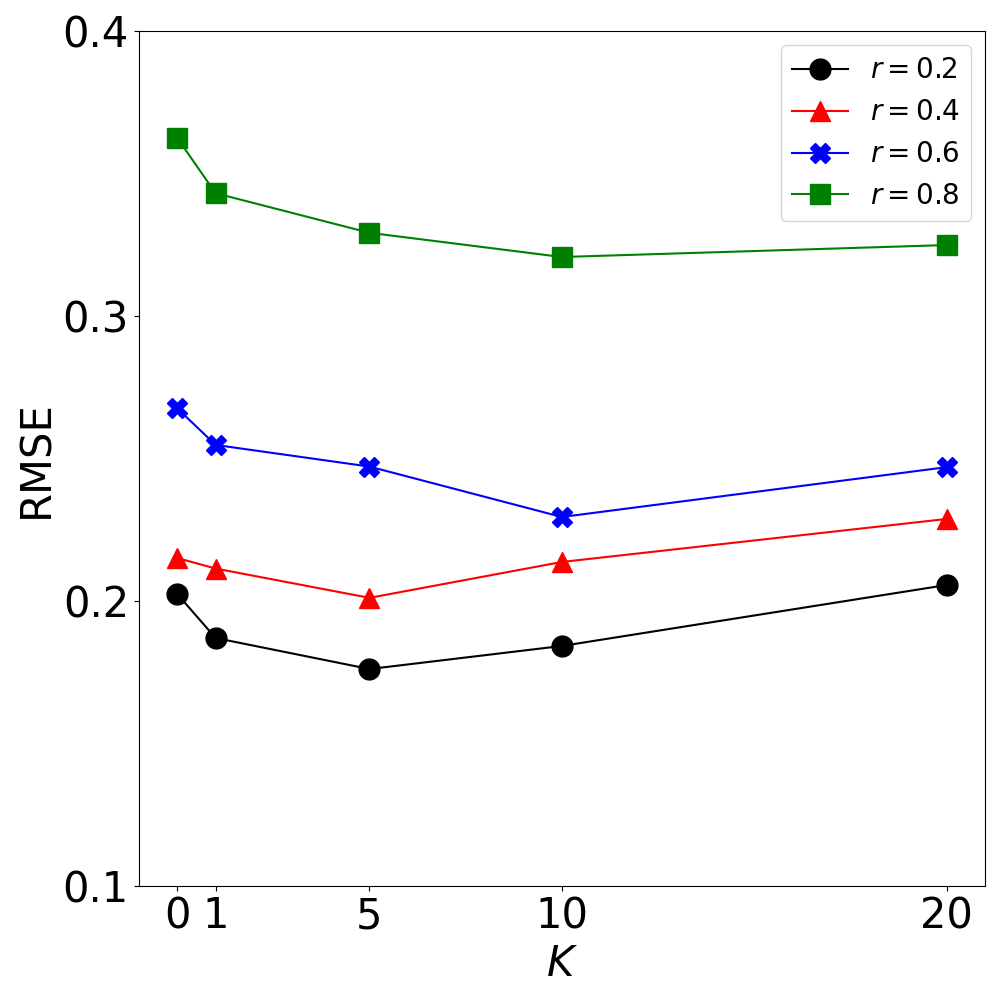}
 \caption{Single Imputation}\label{fig:sensitivity_single_impute}
\end{subfigure}\,
\begin{subfigure}[b]{.23\textwidth}
  \includegraphics[width=\linewidth]{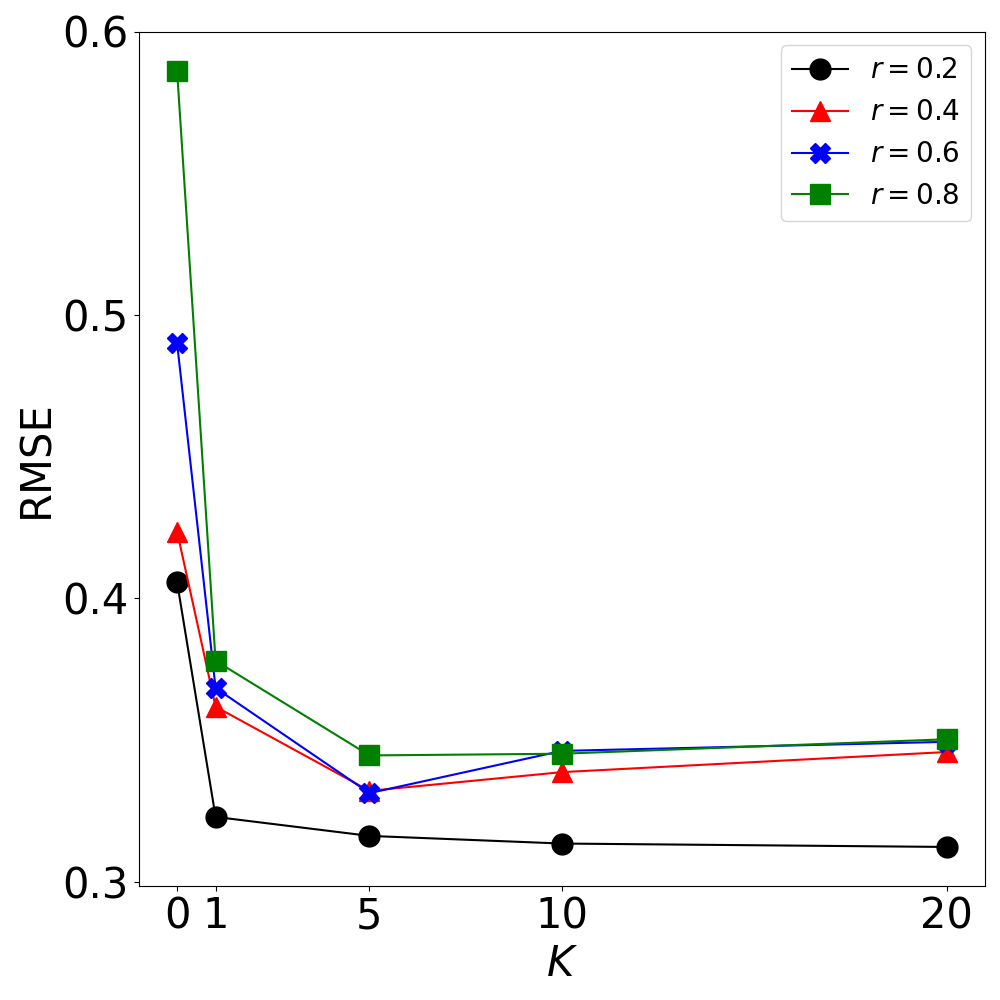}
  \caption{Spatial-Temporal Forecasting}\label{fig:sensitivity_co_forecast}
\end{subfigure}\,
\begin{subfigure}[b]{.23\textwidth}
  \includegraphics[width=\linewidth]{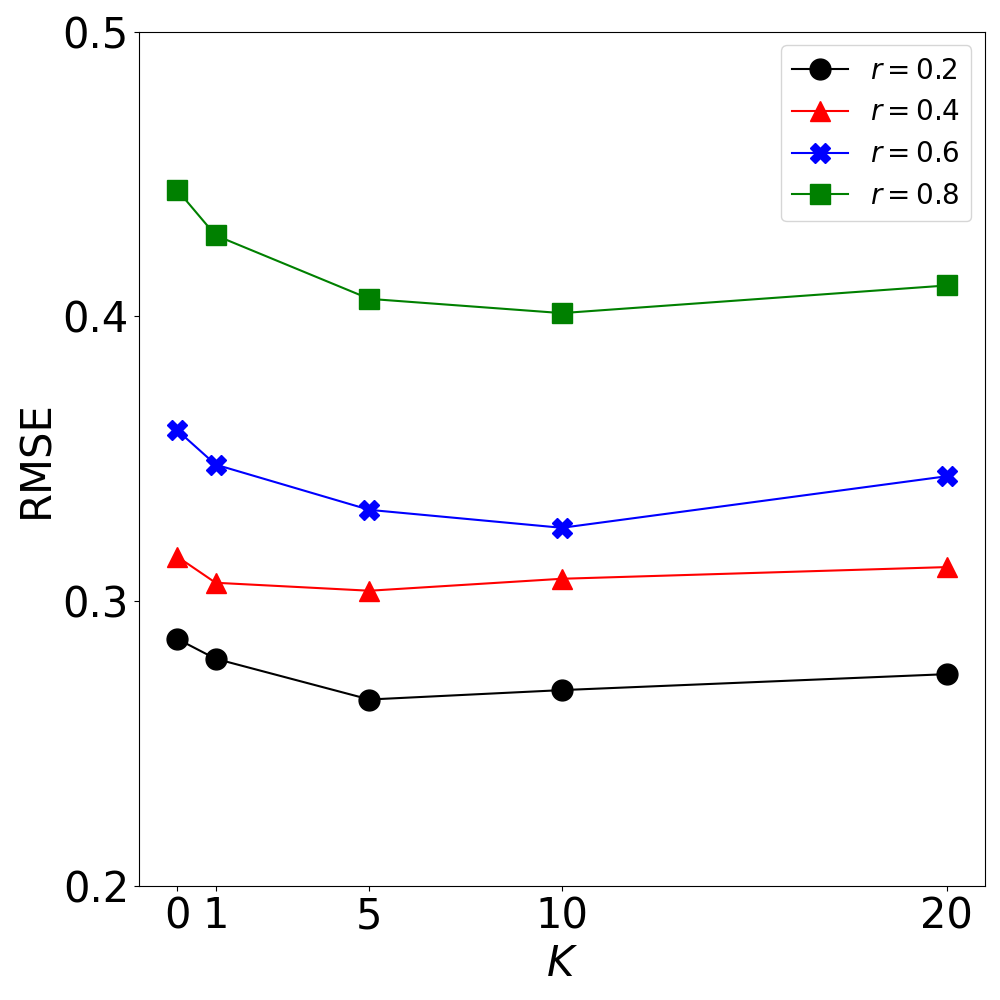}
  \caption{Spatial-Temporal Imputation}\label{fig:sensitivity_co_impute}
\end{subfigure}\,
\caption{The performance of \retime\ w.r.t. different $K$.}
\label{fig:sensitivity}
\end{figure*}

\begin{figure}[t]
\centering
\begin{subfigure}[b]{.22\textwidth}
\includegraphics[width=\linewidth]{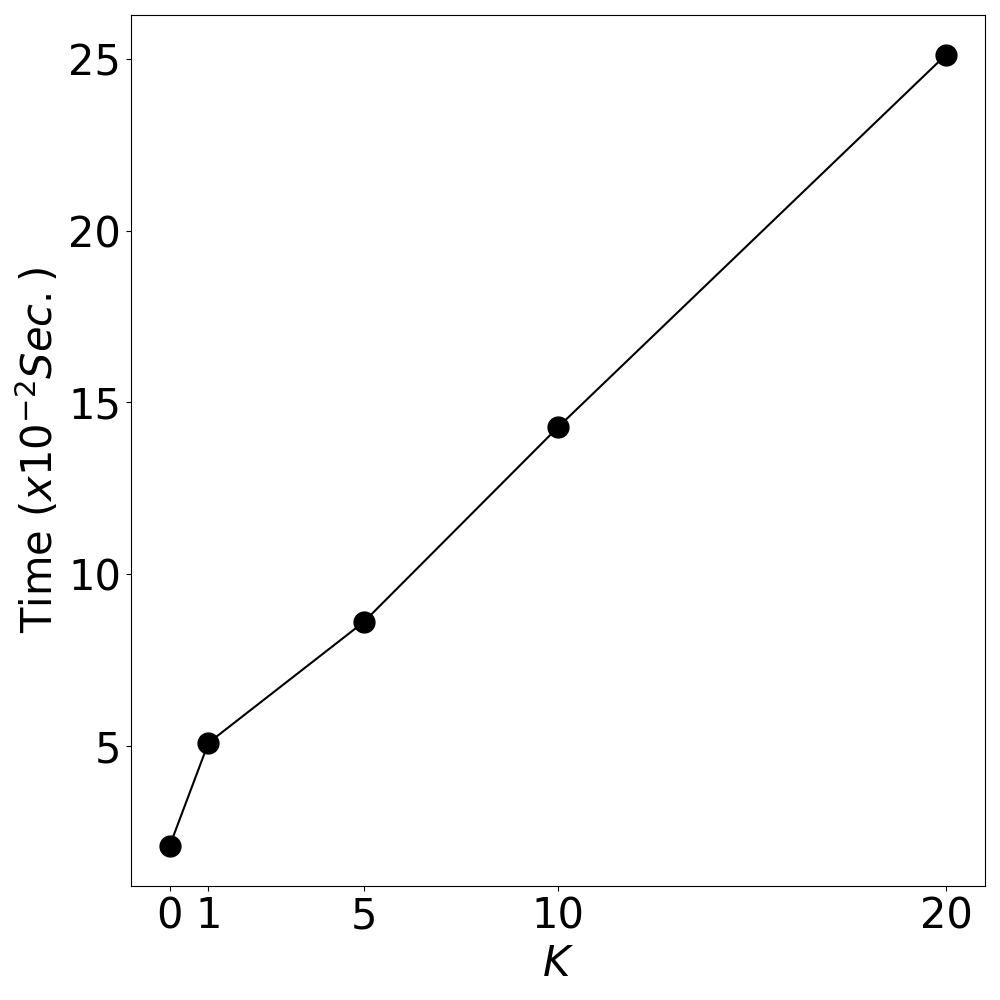}
 \caption{Time Usage}\label{fig:time_vs_k}
\end{subfigure}\,
\begin{subfigure}[b]{.22\textwidth}
  \includegraphics[width=\linewidth]{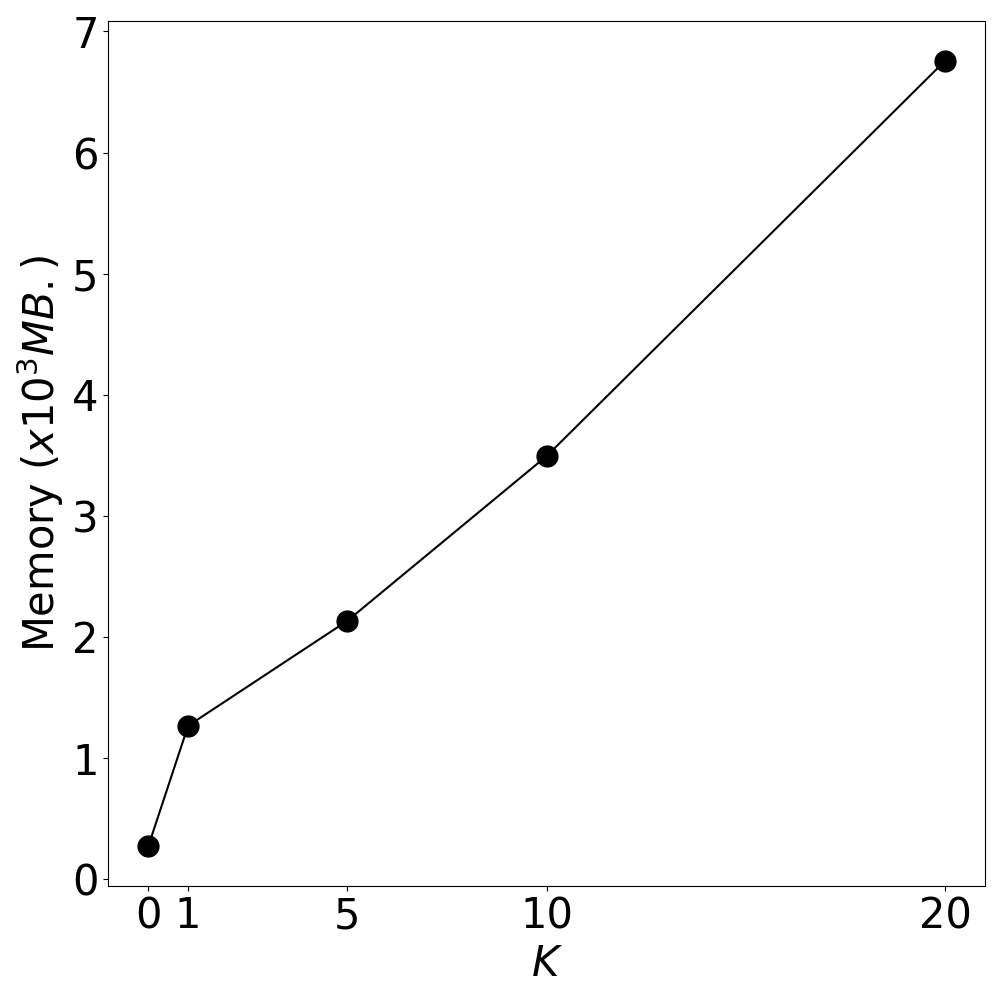}
  \caption{Memory Usage}\label{fig:memory_vs_k}
\end{subfigure}
\caption{Time and Memory Usage.}
\label{fig:time_and_memory}
\end{figure}

\subsection{Impact of $K$}
We study the performance of \retime\ and its time/memory usage w.r.t. the number of references $K\in\{0, 1, 5, 10, 20\}$.
The experiments are conducted on the traffic dataset.

\noindent\textbf{Performance of \retime.}
The performance of \retime\ w.r.t. the number of references $K$ is presented in Figure \ref{fig:sensitivity}, where $r$ denotes the missing rate.
\retime\ achieves the best performance when $K\in\{5, 10\}$
Note that 
for spatial-temporal forecasting, the top 1 reference snippet of a target is its own historical snippet.

\noindent\textbf{Time and Memory Usage.}
We fix the batch size as 100 and record the average training time ($\times 10^{-2}$ seconds) of each iteration and the GPU memory usage ($\times 10^3$ MegaBytes) for different $K$.
The results in Figure \ref{fig:time_and_memory} show that the training time and memory usage grow linearly w.r.t. $K$.


\section{Related Work}\label{sec:related_work}
\subsection{Single Time Series Methods}
Many deep learning methods have been proposed to generate time series, including Recurrent Neural Network (RNN) based methods \cite{cao2018brits, che2018recurrent} and block-style methods \cite{oreshkin2019n, zhou2021informer}.
For example, 
Cao et al. \cite{cao2018brits} introduce BRITS which leverages the recurrent dynamics for both correlated and uncorrelated multivariate time series. 
Fortuin et al. \cite{fortuin2020gp} combines the Gaussian process to capture the temporal dynamics and reconstruct missing values by VAE \cite{kingma2013auto}. 
Luo et al. \cite{luo2018multivariate} introduce GRUI and propose a two-stage GAN \cite{goodfellow2014generative} model, the generator and discriminator of which are based on GRUI. 
E2GAN \cite{luo2019e2gan} further simplifies the generator by combining GRUI and denoising autoencoder such that the GAN-based imputation can be trained in an end-to-end manner. 
Gamboa et al. \cite{gamboa2017deep} explore different neural networks for time series analysis.
Oreshkin et al. \cite{oreshkin2019n} introduce N-BEATS for explainable time series forecasting.
Li et al. \cite{li2019enhancing} propose an enhanced version of Transformer \cite{vaswani2017attention} for forecasting.
Zhou et al. \cite{zhou2021informer} propose Informer for long time series forecasting.
Wu et al. \cite{wu2021autoformer} introduce AutoFormer, which reduces the complexity of Transformer.
However, when the missing rate grows to a high level, the performance of these methods drops significantly.
\retime\ address this issue by retrieving relevant reference time series as an augmentation. 

\subsection{Spatial-Temporal Time Series Methods}
Time series often co-evolve with each other.
Networks/graphs are commonly used data structures to model the relations among objects \cite{tong2006fast, jing2021hdmi, li2022graph, jing2021graph}, which have also been used to model relations in spatial-temporal time series.
Traditional methods leverage probabilistic graphical models to introduce graph regularizations \cite{cai2015fast, cai2015facets}.
Recently, Li et al. \cite{li2017diffusion} introduce DCRNN combining the diffusion graph kernel with RNN.
Yu et al. \cite{yu2017spatio} and Zhao et al. \cite{zhao2019t} respectively propose STGCN and TGCN for modeling spatial-temporal time series.
Jing et al. \cite{jing2021network} introduce NET$^3$ which captures both explicit and implicit relations among time series.
Cao et al. \cite{NEURIPS2020_cdf6581c} introduce StemGNN which combines graph and discrete Fourier transform to jointly model spatial and temporal relations.
Wu et al. \cite{wu2020connecting} introduce MTGNN which learns the relation graph of time series.
However, these methods do not allow models to refer to the historical snippets when forecasting future values, which perform worse than \retime.

\subsection{Retrieval Based Generation}
The main idea underlying retrieval-based methods is to retrieve references from databases to guide generation.
Cao et al. \cite{cao2018retrieve} propose Re$^3$Sum to generate document summaries based on the retrieved templates.
Song et al. \cite{song2018ensemble} propose to generate dialogues based on the retrieved references. 
Lewis et al. \cite{lewis2020retrieval} introduce Retrieval-Augmented Generation (RAG) for knowledge-intensive natural language processing tasks e.g., summarization \cite{jing2021multiplex}.
Tseng et al. \cite{tseng2020retrievegan} propose RetrievalGAN to generate images by retrieving relevant images.
Ordonez et al. \cite{ordonez2016large} generate image descriptions based on the retrieved captions.
To the best of our knowledge, we present the first retrieval-based deep generation model for time series data.

\section{Conclusion}\label{sec:conclusion}
In this paper, we theoretically quantify the uncertainty of the predicted values and prove that retrieved references could help to reduce the uncertainty for prediction results.
To empirically demonstrate the effectiveness of retrieval-based forecasting, we build a simple yet effective method called \retime, which is comprised of a relational retrieval stage and a content synthesis stage.
The experimental results on real-world datasets demonstrate the effectiveness of \retime. 


\bibliographystyle{unsrtnat}
\bibliography{sample-base}

\end{document}